\documentclass{article}

 \usepackage[main, final]{neurips_2026}

\usepackage[utf8]{inputenc} 
\usepackage[T1]{fontenc}    
\usepackage[colorlinks=true,linkcolor=red,citecolor=blue,urlcolor=red]{hyperref}
\usepackage{url}            
\usepackage{booktabs}       
\usepackage{amsfonts}       
\usepackage{nicefrac}       
\usepackage{microtype}      
\usepackage{xcolor}         
\usepackage{aliascnt} 
\usepackage{wrapfig}
\usepackage{nicefrac}
\usepackage{booktabs}
\usepackage{graphicx}
\usepackage{float}
\usepackage{amsmath}
\usepackage{amssymb}
\usepackage{mathtools}
\usepackage{amsthm}
\usepackage{makecell}
\usepackage{placeins}
\usepackage{xspace}
\usepackage{xcolor}
\usepackage{cancel}
\usepackage[capitalize,noabbrev]{cleveref}

\newcommand{\policy}{\ensuremath{\pi(\mathbf{a}|\mathbf{s})}\xspace}
\newcommand{\policyOpt}{\ensuremath{\pi^{\star}(\mathbf{a}|\mathbf{s})}\xspace}

\newcommand{\policyIter}[1]{\ensuremath{\pi^{(#1)}(\mathbf{a}|\mathbf{s})}\xspace}
\newcommand{\rewardIter}[1]{\ensuremath{r^{(#1)}(\mathbf{s},\mathbf{a})}\xspace}

\newcommand{\policyDist}{\ensuremath{\rho_{\pi}(\mathbf{s})}\xspace}
\newcommand{\policyDistOpt}{\ensuremath{\rho_{\pi^{\star}}(\mathbf{s})}\xspace}
\newcommand{\policyDistIter}[1]{\ensuremath{\rho_{\pi^{(#1)}}(\mathbf{s})}\xspace}

\newcommand{\policyOcc}{\ensuremath{\rho_{\pi}(\mathbf{s},\mathbf{a})}\xspace}
\newcommand{\policyOccIter}[1]{\ensuremath{\rho_{\pi^{(#1)}}(\mathbf{s},\mathbf{a})}\xspace}
\newcommand{\policyOccOpt}{\ensuremath{\rho_{\pi^{\star}}(\mathbf{s},\mathbf{a})}\xspace}
\newcommand{\expOcc}{\ensuremath{\rho_{E}(\mathbf{s},\mathbf{a})}\xspace}
\newcommand{\updateOp}[1]{\ensuremath{\mathcal{U}_{\epsilon}^{\rho^{(#1)}}}\xspace}

\newcommand{\policyLR}{\omega}

\theoremstyle{plain}
\newtheorem{theorem}{Theorem}[section]

\newaliascnt{proposition}{theorem}
\newtheorem{proposition}[proposition]{Proposition}
\aliascntresetthe{proposition}
\crefname{proposition}{Proposition}{Propositions}
\Crefname{proposition}{Proposition}{Propositions}

\newtheorem{lemma}[theorem]{Lemma}
\newtheorem{corollary}[theorem]{Corollary}
\theoremstyle{definition}

\newtheorem{assumption}[theorem]{Assumption}
\theoremstyle{remark}
\newtheorem{remark}[theorem]{Remark}

\crefname{lemma}{Lemma}{Lemmas}
\Crefname{lemma}{Lemma}{Lemmas}
\crefname{assumption}{Assumption}{Assumptions}
\Crefname{assumption}{Assumption}{Assumptions}

\theoremstyle{definition}
\newaliascnt{fact}{theorem}
\newtheorem{fact}[fact]{Fact}
\aliascntresetthe{fact}
\crefname{fact}{fact}{facts}
\Crefname{fact}{Fact}{Facts}

\newcounter{thmstep}[theorem]
\renewcommand{\thethmstep}{\thetheorem\Alph{thmstep}}
\crefname{thmstep}{theorem}{theorems}
\Crefname{thmstep}{Theorem}{Theorems}
\newcommand{\ThmStepLabel}[1]{%
  \refstepcounter{thmstep}\label{#1}%
}

\allowdisplaybreaks
\newcounter{assumpstep}[theorem]

\crefname{assumpstep}{assumption}{assumptions}
\Crefname{assumpstep}{Assumption}{Assumptions}

\newcommand{\AssumpStepLabel}[1]{%
  \refstepcounter{assumpstep}%
  \label{#1}%
}

\newenvironment{tightquote}
  {\begin{list}{}{\leftmargin=1em \rightmargin=0pt \topsep=0.3em \itemsep=0pt \parsep=0pt}\item[]}
  {\end{list}}
\newcommand{\coloremph}[1]{\textcolor[HTML]{7B1FA2}{\emph{#1}}}

\title{A Geometric Perspective on Reward Function Updates in Inverse Reinforcement Learning}

\author{%
  Anish Diwan$^{1,2}$ \quad
  Jan Peters$^{1,2,3,4}$ \quad
  Oleg Arenz$^{1,2}$ \\[0.5em]
  $^{1}$Technical University of Darmstadt \quad
  $^{2}$Robotics Institute Germany (RIG) \\
  $^{3}$hessian.AI \quad
  $^{4}$German Research Center for AI (DFKI) \\[0.3em]
  \texttt{anish.diwan@robot-learning.de}
}

\begin{document}

\maketitle

\begin{abstract}
Inverse Reinforcement Learning (IRL) is classically solved using Lagrangian optimisation, by jointly optimising a primal-dual problem with gradient descent to obtain the optimal reward function parameters and corresponding optimal policy. While this algorithmic view of IRL has been highly influential, gradient-based reward updates only paint a partial picture, largely ignoring the geometry of the problem in occupancy space. In this paper, we present an occupancy-space view of reward updates in IRL. We leverage the insight that both the optimal solution and its geometry in occupancy space are known, and show that the geometry along the e-geodesic (straight line in occupancy space) toward the optimal occupancy directly affects the gradient of the IRL dual problem in reward parameter space. This perspective naturally suggests preconditioning the dual gradient with the average Fisher information along the occupancy trajectory. We show that this preconditioned gradient is just a convex interpolation of the old reward and a new target reward, yielding efficient natural-gradient-style updates with little additional computational overhead. We then analyse these occupancy-space reward updates from a Mirror Descent (MD) perspective, and show that policy search over the resulting rewards induces approximate entropic MD on the policy. Finally, we conclude with convergence results of such geometry-aware IRL algorithms\footnote{Our code is available at \href{https://github.com/anishhdiwan/geometric-perspective}{https://github.com/anishhdiwan/geometric-perspective}.}.
\end{abstract}

\section{Introduction}
\label{sec:introduction}
Inverse Reinforcement Learning (IRL) \citep{russell1998learning,ng2000algorithms} is the problem of inferring a reward function under which some observed behaviour is optimal. IRL is classically formulated as entropy-regularized distribution matching, where the goal is to maximise causal entropy while simultaneously minimising the divergence between the occupancy measures of the expert and the learning agent, $\rho_E$ and $\rho_\pi$. We are interested in the specific problem setting of entropy-regularized reverse-KL divergence minimisation between the expert's and agent's occupancy measures. In this setting, we start with the following observation:
\begin{tightquote}
\textit{The optimal solution to the IRL problem admits an explicit characterization in the space of occupancy measures. An approximation of this optimal occupancy ($\rho_{\hat{\pi}}$) \textbf{is known} based on the current reward function. Reward updates in IRL can be seen as facilitating $\rho_\pi \rightarrow \rho_{\hat{\pi}}$.}
\end{tightquote}
\begin{wrapfigure}{r}{0.5\columnwidth}
    \vspace{-0.5em}
    \centering
    \includegraphics[width=\linewidth]{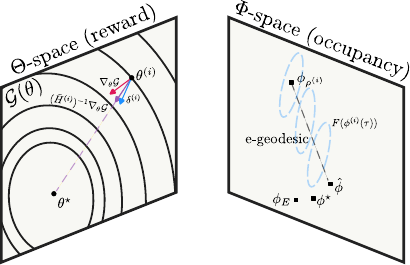}
    \caption{We illustrate the geometry of the IRL problem by showing the dual of \Cref{eq:IRLDM} $\mathcal{G}(\theta)$ in reward parameter space, and highlighting the \textcolor[HTML]{DF1659}{vanilla gradient}, the ideal \textcolor[HTML]{9B59B6}{one-step optimisation direction}, and our \textcolor[HTML]{1E90FF}{integral FIM-preconditioned gradient}. The vanilla gradient is normal to the contour of $\mathcal{G}(\theta)$ and potentially leads to poor reward updates. In contrast, our FIM-preconditioned direction considers the geometry along the e-geodesic in occupancy space, leading to geometry-aware reward updates.}
    \label{fig:hero}
    \vspace{-0.5em}
\end{wrapfigure}
Thus, rather than treating the IRL reward update as a generic gradient descent step on a parametrised objective, we can view it as an attempt to move the current policy occupancy $\rho_\pi$ towards an estimate of the optimal one, $\rho_{\hat{\pi}}$. This changes the nature of the analysis. The central object is no longer only the reward parameter error, but the induced discrepancy in occupancy parameters, which measures progress in the space where the optimum is explicitly characterized. The non-trivial aspect of IRL, however, is that occupancy-space updates are not directly realizable and this change must be induced through policy optimisation on an updated reward function. Consequently, the geometry of the occupancy parameter space and the mapping from occupancy update directions to reward function gradients play a crucial role in the effectiveness of reward function updates in IRL. 

The goal of this paper is to make this geometric structure explicit (illustrated in \Cref{fig:hero}). We show that the gradient-based reward update direction can be greatly improved by preconditioning the gradient with the average Fisher Information Matrix (FIM) along the e-geodesic (occupancy trajectory) of the IRL problem. We analyse the effectiveness of this average FIM preconditioning by relating it to an average Hessian preconditioner, which admits the best-possible (closed form) solution with stepsize 1. We argue that preconditioning with this average FIM yields a natural gradient-style reward function update, and show that this update is equivalent to a convex interpolation of the previous reward function and a new estimate of the optimal reward. Hence, the proposed FIM preconditioned reward update resembles function-space reward updates \citep{arenz2016optimal,diwan2026trust}, and can be realised efficiently without the explicit computation of the FIM \citep{peters2008natural}. Finally, we analyse such geometry-aware reward updates from a Mirror Descent (MD) \citep{nemirovskij1983problem,beck2003mirror} perspective, and show that  policy search over them is equivalent to entropic MD. Our interpretation hence connects inverse reinforcement learning, natural gradient descent, and mirror descent. By centering the analysis on the known occupancy-space optimum, we obtain a more transparent view of reward learning: \coloremph{reward updates are best understood not simply as steps in reward parameter space, but as geometry-induced movements toward an optimal occupancy representation}.
\paragraph{Original Contributions.}
We present a rigorous theoretical analysis of reward learning in IRL through entropy-regularized reverse-KL distribution matching. Our main result is that translating the occupancy space updates (towards a known distribution-matching optimum) into reward space naturally admits geometry-awareness and a reward update strategy that is often better than simple gradient descent on reward parameters. Our main contributions are as follows:
\begin{itemize}
    \item In \Cref{sec:occupancy_space_rew_updates}, we show that the occupancy-space update $\rho_\pi \rightarrow \rho_{\hat{\pi}}$ translates to a convex interpolation of reward functions in reward parameter space
    \begin{align}
        ( \updateOp{i} ) r^{(i)} \equiv r^{(i+1)} = r^{(i)} - \epsilon \delta^{(i)} \label{eq:functional_update}
    \end{align}
    where $\delta^{(i)} = \theta^{(i)} - \beta \log (\nicefrac{\rho_E}{\rho_{\pi^{(i)}}})$ is the update direction. We use the operator $( \updateOp{i} ) r^{(i)}$ to define this interpolation and call it a ``function-space reward update". We note that such updates have been previously studied by \citep{arenz2016optimal,diwan2026trust} and have been empirically shown to be far superior to gradient updates. Our geometric interpretation is the first to explain \textit{why} such function-space updates are effective. We show that
    \begin{itemize}
        \item[(i)] with minimal assumptions, $\delta^{(i)}$ can be interpreted as a gradient preconditioned by the average Fisher Information over the occupancy trajectory, resulting in a natural-gradient style reward update.
        \item[(ii)] $( \updateOp{i} ) r^{(i)}$ is equivalent to Mirror Descent (MD) with a time-varying proximal term.
    \end{itemize}
    \item In \Cref{sec:trirl_analysis} we present a detailed analysis of policy search on such geometry-aware reward updates, showing that it is equivalent to entropic Mirror Descent. Hence, an IRL algorithm built on functional reward updates has a unique geometric interpretation: the policy update performs MD in occupancy space, while the reward update performs MD in the geometry induced by the reverse-KL distribution-matching problem.
    \item In \Cref{sec:convergence_analysis} we present convergence results for such IRL methods.
\end{itemize}

\section{Background}
\label{sec:background}
\paragraph{Notation.} Following standard notation, we define $\Delta_{X}$ as the set of probabilities over $X$ and $\Delta_{X}^{Y}$ as the set of conditional probabilities over $X$ conditioned on $Y$. For two vectors $\mathbf{x},\mathbf{y} \in\mathbb{R}^d$, we denote $\langle \mathbf{x},\mathbf{y} \rangle$ as their inner product. A function $f:\mathbb{R}^d\to\mathbb{R}$ is $L_1$-Lipschitz continuous if $\left \lvert f(\mathbf{x})-f(\mathbf{x'}) \right\rvert \leq L_1 \left\lVert \mathbf{x} - \mathbf{x'} \right\rVert_2$, and it is $L_2$-smooth if it is differentiable and $\left \lVert \nabla_\mathbf{x} f(\mathbf{x})-\nabla_\mathbf{x} f(\mathbf{x'}) \right\rVert_2 \leq L_2 \left\lVert \mathbf{x} - \mathbf{x'} \right\rVert_2$ for all $\mathbf{x} , \mathbf{x'} \in\mathbb{R}^d$. For some function $f$, we denote $f^{(i)}$ for the function at iteration ${(i)}$ and $f^{\text{conj}}$ as the function's convex conjugate. 
\paragraph{Preliminaries.} We model the environment as a Markov Decision Process (MDP) defined by a tuple $(\mathcal{S}, \mathcal{A}, \mu_0, \mathcal{P}, r, \gamma)$, where $\mathcal{S}$ is the state space, $\mathcal{A}$ is the action space, $\mu_0$ is the initial state distribution, $\mathcal{P} \in \Delta_{\mathcal{S}}^{\mathcal{S} \times \mathcal{A}}$ represents the transition dynamics, $r \in \mathbb{R}^{\mathcal{S} \times \mathcal{A}}$ is the (unknown) reward function, and $\gamma \in [0,1)$ is the discount factor. $\Pi \triangleq \Delta_{\mathcal{A}}^{\mathcal{S}}$ is the set of all stationary stochastic policies mapping states in $\mathcal{S}$ to actions in $\mathcal{A}$. We define the normalized discounted occupancy measure \footnote{In some cases, we use $\rho^{(i)}$ as shorthand for $\rho_{\pi^{(i)}}$.} $\policyOcc = (1 - \gamma) \pi(\mathbf{a} \vert \mathbf{s}) \sum_{t=0}^\infty \gamma^t \mu^\pi_t(\mathbf{s})$, where $\mu_{t+1}^\pi(\mathbf{s}')= \sum_{\mathbf{s}} \mu^\pi_{t}(\mathbf{s}) \sum_{\mathbf{a}} \pi(\mathbf{a} \vert \mathbf{s})\mathcal{P}(\mathbf{s}' \vert \mathbf{s}, \mathbf{a})$ is the state distribution for $t>0$, with $\mu^\pi_0(\mathbf{s})=\mu_0(\mathbf{s})$. We work in the $\gamma$-discounted infinite horizon setting and use an expectation with respect to a policy $\pi \in \Pi$ to denote an expectation with respect to the trajectory it generates; $\mathbb{E}_{\pi}[r(\mathbf{s}, \mathbf{a})] \triangleq \mathbb{E}_{(\mathbf{s}_0\sim \mu_0, \mathbf{a}_t \sim \pi, \mathbf{s}_{t+1} \sim \mathcal{P})} [\sum_{t=0}^\infty \gamma^{t} r(\mathbf{s}_t, \mathbf{a}_t)]$. We refer to the (unknown) expert policy as $\pi_E$ and its occupancy measure as $\expOcc$.
\paragraph{Entropy-Regularized Inverse RL.} Given a set of demonstration trajectories $\mathcal{D} = \{ (\mathbf{s}_i,\mathbf{a}_i) \}_{i=1}^N$ sampled from an expert, MaxCausalEnt-IRL \citep{ziebart2010modeling} aims to find a reward function $r \in \mathbb{R}^{\mathcal{S} \times \mathcal{A}}$ and a policy $\pi \in \Pi$ that solve the optimisation problem: $\max_{\pi} \min_{r} \left( \mathbb{E}_{\rho_\pi}[r(\mathbf{s}, \mathbf{a})] + H(\pi) \right) - \mathbb{E}_{\rho_{E}}[r(\mathbf{s}, \mathbf{a})]$. Here $H(\pi) \triangleq \mathbb{E}_{\pi}[-\log \pi(\mathbf{a} \vert \mathbf{s})]$ is the discounted causal entropy of the policy in state $\mathbf{s}$. On formulating and solving the problem's Lagrangian, the optimal policy satisfies the Boltzmann distribution $\pi^{\star}(\mathbf{a} \vert \mathbf{s}) = \nicefrac{1}{Z(\mathbf{s})}\exp{(Q^{\star}(\mathbf{s}, \mathbf{a}))}$ where $V^{\star}(\mathbf{s}) = \log \sum_{\mathbf{a}'} \exp(Q^{\star}(\mathbf{s}, \mathbf{a}'))$ and $Q^{\star}(\mathbf{s}, \mathbf{a}) =  r(\mathbf{s}, \mathbf{a}) + \gamma \mathbb{E}_{\mathbf{s}' \sim \mathcal{P}} \left[\log \sum_{\mathbf{a}'} \exp(Q^{\star}(\mathbf{s}', \mathbf{a}')) \right]$ are the optimal soft value functions and $Z(\mathbf{s})$ is a normalization term.

\label{sec:distribution_matching}
\paragraph{IRL by Reverse-KL Divergence Minimisation.}
Our analysis focuses on the setting of inverse RL by entropy-regularized \footnote{Strictly speaking, in MaxCausalEnt-IRL the causal entropy is the objective rather than a regulariser; regularisation arises from relaxing the feature-matching constraint. Common relaxations bound the absolute or Euclidean feature-matching error, which corresponds to $\ell_1$ or $\ell_2$ penalties on the reward parameters in the dual~\citep{dudik2007maximum}. We instead relax the constraint via a reverse-KL divergence between occupancies. Rather than penalising the reward parameters, this relaxation modifies the target distribution: the agent no longer matches $\rho_E$ but the reward-dependent target $\rho_{\hat{\pi}}$.} reverse-KL divergence minimisation between the expert's and agent's state-action occupancies $\policyOcc$ and $\expOcc$. Divergence minimisation has been central to several prior works like \citep{fu2017learning,kostrikov2019imitation,arenz2020non,al2023ls,watson2023coherent,diwan2026trust} and is represented by the following objective:
\begin{align}
   \label{eq:IRLDM}
\mathcal{J}(\policy) &= \mathbb{E}_{\policyDist} \left[ H(\policy) \right] - \beta \mathbb{E}_{\policyOcc} \left[ \log \frac{\policyOcc}{\expOcc} \right]
\end{align}
where $\beta$ trades off between entropy regularization and the objective of matching the expert's occupancy. This problem is typically solved using Lagrangian optimisation, by formulating its Lagrangian $\mathcal{L}(\pi, r)$, deriving the Lagrangian dual $\mathcal{G}(r)$, and obtaining the optimal solution by setting derivatives to zero. It can be shown that, in the limit $\beta \rightarrow \infty$, the optimal policy and value functions are the same as for MCE-IRL, and the optimal reward function is $r^{\star} = \beta \log \left( \nicefrac{\expOcc}{\rho_{\pi^\star}(\mathbf{s},\mathbf{a})} \right)$, where $(\pi^\star, r^\star)$ is the same saddle point as in MCE-IRL \citep[Appendix I,II]{arenz2016optimal}. Primal-dual algorithms like \citet{ziebart2010modeling} find this optimal solution by alternating between optimising the primal variables to get $\pi_r$, and taking gradient descent steps on the dual variable $r$.
\section{Occupancy-Space Reward Function Updates}
\label{sec:occupancy_space_rew_updates}
We are now ready to formalize the occupancy-space update $\rho_\pi \rightarrow \rho_{\hat{\pi}}$. We begin by stating our central modelling assumption.
\begin{assumption}[EF Occupancies]
\label[assumption]{assumption:boltzmann_occupancies}
The expert and agent occupancies are exponential family distributions with energy functions that are linear in features $\psi(\mathbf{s}, \mathbf{a})$. Such that, $\expOcc = Z^{-1}_{E} \exp{\left( \phi_{E}^\intercal \; \psi(\mathbf{s}, \mathbf{a}) \right)}$ and $\policyOccIter{i} = Z^{-1}_{\rho^{(i)}} \exp{\left( \phi_{\rho^{(i)}}^\intercal \; \psi(\mathbf{s}, \mathbf{a}) \right)}$.
\end{assumption}
Where, $\phi_{E} , \phi_{\rho}$ $\in \Phi \triangleq \mathbb{R}^{d}$ are the natural parameters of the exponential family distribution. \Cref{assumption:boltzmann_occupancies} is borrowed from \citet{arenz2016optimal} and allows us to analyse expressive function-space updates. While we do not assume the features $\psi(\mathbf{s}, \mathbf{a})$ to be known, in \Cref{appendix:assumptions} we also provide a method to learn them from a finite demonstration dataset. \Cref{assumption:boltzmann_occupancies} directly implies that the reward function is linear in features as $r^{(i)} = \theta^{(i) \intercal} \psi(\mathbf{s}, \mathbf{a})$, where $\theta \in \Theta \triangleq \mathbb{R}^{d}$ are reward parameters (hence, we can interchange $\mathcal{G}(r)$ and $\mathcal{G}(\theta)$). In this paper, we refer to $\Theta$ as the \coloremph{reward parameter space} and $\Phi$ as the \coloremph{occupancy parameter space}. From \Cref{sec:distribution_matching} we know the optimal reward function $r^{\star} = \beta \log \left( \nicefrac{\expOcc}{\rho_{\pi^\star}(\mathbf{s},\mathbf{a})} \right)$. From this, and the current reward iterate $r^{(i)}$, it follows that:
\begin{align}
\label{eq:phi_hat}
    \rho_{\hat{\pi}}(\mathbf{s}, \mathbf{a}) \propto \exp{\left( \log \expOcc - \frac{1}{\beta} r^{(i)} \right)} &= \exp{\left( \left(\phi_{E} - \frac{\theta^{(i)}}{\beta}\right)^\intercal \psi(\mathbf{s}, \mathbf{a}) \right)} = \exp{\left(\hat{\phi}^\intercal \psi(\mathbf{s}, \mathbf{a}) \right)} \nonumber \\
    &\text{\textit{where,}} \quad \hat{\phi} = \left(\phi_{E} - \frac{\theta^{(i)}}{\beta}\right)
\end{align}
and $\rho_{\hat{\pi}}$ is an estimate of the optimal solution based on the current reward function, and serves as an effective target occupancy---similar to $\rho_{E}$ in the unregularized IRL setting. From \Cref{assumption:boltzmann_occupancies} we also get an alternate definition of $\delta^{(i)}$ (from \Cref{eq:functional_update})
\begin{align*}
    \beta \log \frac{\expOcc}{\policyOccIter{i}} &= \beta \left( \phi_{E} -  \phi_{\rho^{(i)}} \right)^\intercal \psi(\mathbf{s}, \mathbf{a})  + \textrm{const.} \\
    \text{\textit{Substituting}} &\text{\textit{ in \Cref{eq:functional_update} and writing in reward parameters}} \\
    \delta^{(i)} &= \theta^{(i)} - \beta \left( \phi_{E} -  \phi_{\rho^{(i)}} \right) \nonumber \\
    &= \beta \left( \phi_{\rho^{(i)}} - \hat{\phi} \right).
\end{align*}
Here, $\hat{\phi} = \phi_{E} - \left( \nicefrac{\theta^{(i)}}{\beta} \right)$ is the natural parameter of $\rho_{\hat{\pi}}$ and the optimisation target in $\Phi$ space. In other words, the optimisation target in occupancy-parameter space is fully determined by the expert features $\phi_{E}$ and the current reward $\theta^{(i)}$. Our occupancy space interpretation of IRL is that reward updates must facilitate policies that indeed aim to push $\phi_{\rho^{(i)}}$ towards $\hat{\phi}$. However, as highlighted before, a direct step in $\Phi$ space is non-trivial since there is no reliable way to obtain policies from occupancies. Instead, we (i) show that the direct $\Phi$-space direction $\phi_{\rho^{(i)}} - \hat{\phi}$ is equivalent to a geometry-aware gradient in $\Theta$-space (resembling the natural gradient \citep{amari1998natural,peters2008natural}), (ii) the geometry-aware gradient in $\Theta$-space is in-turn equivalent to a convex interpolation of reward functions and is hence directly computable without the approximation of hessians/FIMs, (iii) this geometry-aware reward update is often better than just using the standard parametric gradient. 

\subsection{Assumptions}
\label{sec:reward_assumptions}
Before diving further into the study of $\Phi$-space reward updates, we first list and motivate the properties and assumptions necessary for our analysis. 
\begin{assumption}[Lipschitz \& Bounded Dual]
\label[assumption]{assumption:lipschitzness_and_boundedness}
$\mathcal{G}$ is $L_{\mathcal{G}}$-Lipschitz Continuous on the set of feasible reward parameters $\Theta$. $\mathcal{G}$ is lower bounded.
\begin{align*}
    \left\lvert \mathcal{G}(\theta) - \mathcal{G}(\theta') \right\rvert \leq L_{\mathcal{G}} \left\lVert \theta - \theta' \right\rVert_2 \quad \forall \; \theta, \theta' \in \Theta
\end{align*}
\end{assumption}
\begin{assumption}[Curvature]
\label[assumption]{assumption:covariance_conditions}
At iteration $i$, for $\tau \in [0,1]$, let $\phi^{(i)}(\tau) \triangleq \hat{\phi} + \tau (\phi_{\rho^{(i)}} - \hat{\phi})$ be the optimisation trajectory in $\Phi$-space (along the e-geodesic between $\phi_{\rho^{(i)}}$ and $\hat{\phi}$), ${\text{Cov}}_{\rho_{\phi}}[.]$ be the covariance matrix, and there be constants $0 < m \leq M < \infty$. Then,
\AssumpStepLabel{assumption:nonflat_curvature}
\AssumpStepLabel{assumption:uniform_covariate_conditioning}
\[
\begin{array}{@{}ll@{\qquad}ll@{}}
\textnormal{\textbf{(A)}} &
\underset{\rho_{\phi^{(i)} (\tau)}}{\text{Cov}}[\psi(\mathbf{s}, \mathbf{a})] \succ 0,
&
\textnormal{\textbf{(B)}} &
mI \preceq \underset{\rho_{\phi^{(i)} (\tau)}}{\text{Cov}}[\psi(\mathbf{s}, \mathbf{a})] \preceq MI .
\end{array}
\]
\end{assumption}
\Cref{assumption:nonflat_curvature} is fairly common in RL \citep{kakade2001natural,peters2008natural,liu2020improved,neu2021online} and implies that the feature covariance along the path between $\phi_{\rho^{(i)}}$ and $\hat{\phi}$ is non-singular and hence the curvature of $\Phi$ is non-flat in the region in which we are interested. It also implies that ${\text{Cov}}_{\rho_{\phi}}[.]$ is symmetric positive definite, and hence invertible. \Cref{assumption:uniform_covariate_conditioning} is a stronger version of \Cref{assumption:nonflat_curvature} and is also fairly common in prior works \citep[Assumptions~2.1,5.2,2,6]{liu2020improved,zhang2020global,viano2022proximal,chen2022sample}. It can be interpreted as a bound on the eigenvalues of ${\text{Cov}}_{\rho_{\phi}}[.]$. \Cref{assumption:uniform_covariate_conditioning} is not necessary in all cases \footnote{Used in \Cref{prop:first-ord-accurate,theorem:rewardRegret,sec:connection_to_pseudo_gradient}.} and, unless stated otherwise, we default to \Cref{assumption:nonflat_curvature}.
\subsection{Reward Update Analysis}
\label{sec:reward_analysis}
This section presents our first set of main results. We connect the $\Phi$-space update, $\phi_{\rho^{(i)}} - \hat{\phi}$ to a variant of natural gradient descent, then show that it translates to a function-space update in $\Theta$-space, and quantify its benefits and compare it with standard gradient descent updates. Finally, we show that this $\Phi$-space reward update is equivalent to Mirror Descent with a time-varying proximal term. We start by characterizing the one-step direct optimisation of $\mathcal{G}$, which is the ideal update direction in $\Theta$-space.
\begin{proposition}[Direct Dual Optimisation]
\label[proposition]{prop:direct-dual-opt}
The reward-parameter space direct optimisation step is given by
\begin{align*}
    \theta^{\star} = \theta^{(i)} - \left(\bar{H}^{(i)}\right)^{-1} \nabla_{\theta} \mathcal{G}
\end{align*}
where $\bar{H}^{(i)} = \int_0^1 \nabla^2_{\theta} \mathcal{G}(\theta^{(i)}(\tau)) d\tau \succ 0$, and $\theta^{(i)}(\tau) \triangleq \theta^{\star} + \tau \left(\theta^{(i)} - \theta^{\star} \right)$ is a line in the space of reward parameters from the globally optimal parameter to the current parameter.
\end{proposition}
$\bar{H}^{(i)}$ is the hypothetical ideal preconditioner, which cannot be computed directly since it requires knowledge of the optimal parameter $\theta^{\star}$. The remainder of this section derives a tractable surrogate using only quantities along the known optimisation trajectory toward $\hat{\phi}$. Finally, \Cref{sec:direct_optimisation} quantifies how close this surrogate gets to direct optimisation.
\begin{proposition}[Vanilla Gradient]
\label[proposition]{prop:vanillaGradient}
Let $\mu(\phi) \triangleq \mathbb{E}_{\rho_{\phi}} \left[ \psi(\mathbf{s}, \mathbf{a}) \right]$ be the mean map for some occupancy determined by natural parameter $\phi$, where $A(\phi) = \log Z_{\rho_{\phi}} = \log \sum_{\mathbf{s},\mathbf{a}} \exp{(\phi^\intercal \; \psi(\mathbf{s}, \mathbf{a}))}$ is the log partition function and $F(\phi)$ is the Fisher Information Matrix. Then,
\[
\begin{gathered}
\mu(\phi) = \nabla_{\phi} A(\phi), \qquad \nabla_{\phi} \mu(\phi) = \nabla_{\phi}^2 A(\phi) = \underset{\rho_{\phi}}{\text{Cov}}[\psi(\mathbf{s}, \mathbf{a})] = F(\phi)
\\
\nabla_{\theta} \mathcal{G} = \mathbb{E}_{\policyOccIter{i}} \left[ \psi(\mathbf{s}, \mathbf{a}) \right] - \mathbb{E}_{\rho_{\hat{\pi}}} \left[ \psi(\mathbf{s}, \mathbf{a}) \right] = \mu(\phi_{\rho^{(i)}}) - \mu(\hat{\phi})
\end{gathered}
\]
\end{proposition}
\begin{remark}
\label{remark:geodesics}
\Cref{assumption:boltzmann_occupancies,prop:vanillaGradient} together imply that, in the dually-flat geometry of the exponential family distribution, the vanilla gradient $\nabla_{\theta} \mathcal{G}$ is an m-geodesic (in expectation parameters) whereas the occupancy-space direction $\phi_{\rho^{(i)}} - \hat{\phi}$ is an e-geodesic (in natural parameters).
\end{remark}
Having determined this $\Theta$-space ``vanilla" gradient, we now wish to characterize the $\Theta$-space geometry to better identify our surrogate for the average Hessian from \Cref{prop:direct-dual-opt}. We start by considering an appropriate notion of distance in this space. Rather than considering the standard Euclidean norm, we consider the more general Mahalanobis norm $\left\lVert \mathbf{x} \right\rVert_{H} = \sqrt{\mathbf{x}^\intercal H \mathbf{x}}$ where $H \in \mathbb{R}^{d \times d}$ is a symmetric positive definite matrix (covariance matrix) determining the geometry of the space. Setting $H = I$ returns the Euclidean norm (we will shortly clarify the role of $H$). We then define a quadratic mirror map $\Omega^{(i)}(\theta) : \Theta \to \mathbb{R}$ to recover this norm as a Bregman divergence.
\begin{proposition}[Bregman Div.\ \& Mahalanobis Norm]
\label[proposition]{prop:mahaBregman}
Given some potentially time-varying symmetric positive definite matrix $H^{(i)}$, the Bregman divergence under a quadratic mirror map $\Omega^{(i)}(\theta) \triangleq \nicefrac{\left< \theta , H^{(i)} \theta \right>}{2 \beta}$ is the scaled squared Mahalanobis norm.
\end{proposition}
\Cref{prop:mahaBregman} shows that the proximal term in our subsequent Mirror Descent update is just the squared Mahalanobis norm. Having defined this proximal term, we now (i) define the Mahalanobis norm's covariance matrix $H^{(i)}$ in our specific case and derive the $\Phi$-space update $\phi_{\rho^{(i)}} - \hat{\phi}$ in terms of $H^{(i)}$, and (ii) show that this update indeed performs mirror descent.
\begin{lemma}[Preconditioned Gradient]
\label[lemma]{lemma:natural_gradient}
Let $H^{(i)} \triangleq \int_0^1 F(\phi^{(i)}(\tau)) d\tau$ be the integral Fisher Information Matrix (FIM) over $\phi^{(i)} (\tau)$, the optimisation trajectory in $\Phi$-space. Then,
\begin{align}
    \delta^{{(i)}} = \beta \left( \phi_{\rho^{(i)}} - \hat{\phi} \right) &= \beta \left( H^{(i)} \right)^{-1} \nabla_{\theta} \mathcal{G} \label{eq:delta_nablaG}
\end{align}
\end{lemma}
where we term the $\beta$-scaled function-space update direction as $\delta^{{(i)}}$ (from \Cref{eq:functional_update}). \Cref{lemma:natural_gradient} relates $\delta^{(i)}$ to the gradient of the dual of \Cref{eq:IRLDM} $\nabla_{\theta} \mathcal{G}$ through the average FIM along the line from $\phi_{\rho^{(i)}}$ to $\hat{\phi}$. This implies that the FIM preconditioned gradient can directly be computed as a convex interpolation of the previous reward function and a new estimate of the optimal reward (i.e. the log-density-ratio $\beta \log (\nicefrac{\rho_E}{\rho_{\pi^{(i)}}})$). This trick allows us to effectively compute natural-gradient-like update directions \emph{without explicitly computing the Hessian/FIM}. We elaborate on this in \Cref{sec:natural_gradient}. We now, leverage this definition of $\delta^{(i)}$ and the Bregman divergence from \Cref{prop:mahaBregman} to show that a mirror descent update in $\Theta$-space is equivalent to $( \updateOp{i} ) r^{(i)}$.
\begin{theorem}[Reward-MD]
\label{theorem:rewardMD}
Let $r^{(i+1)} = ( \updateOp{i} ) r^{(i)} =  r^{(i)} - \epsilon \delta^{(i)}$ be a reward update where $\delta^{(i)}$ is the scaled $\Phi$-space update direction. $( \updateOp{i} ) r^{(i)}$ is equivalent to mirror descent with a time-varying, squared Mahalanobis proximal term.
\end{theorem}
Interestingly, the function-space reward update $( \updateOp{i} ) r^{(i)}$ has previously been studied by \citet{arenz2016optimal} and shown to be significantly better than standard gradient steps in reward parameters. \Cref{theorem:rewardMD} provides the first geometric explanation for this empirical superiority: \coloremph{function-space reward updates are Mirror Descent under a time-varying Mahalanobis proximal induced by an average Fisher metric along the e-geodesic to the known optimum in occupancy-space}. The resulting problem also fits the analytical setting of AdaGrad-style time-varying mirror maps \citep[Proposition~3]{duchi2011adaptive,orabona2015generalized} which we leverage in \Cref{sec:convergence_analysis}.
\begin{fact}[Theorem~1 of \citet{arenz2016optimal}]
\label{fact:update_alignment}
The update direction $\delta^{(i)}$ aligns with the ``vanilla" gradient $\nabla_{\theta} \mathcal{G}$.
\end{fact}
\Cref{fact:update_alignment} shows that $\delta^{(i)}$ is a sound update direction and $\mathcal{G}$ improves monotonically on the repeated application of $( \updateOp{i} ) r^{(i)}$. Without making any further assumptions, this also implies that for a sufficiently small step size, an IRL method with optimal primal optimisation converges asymptotically when using $( \updateOp{i} ) r^{(i)}$ for the reward update. We now elaborate on this connection by analysing $H^{(i)}$, the integral FIM over the $\Phi$-space trajectory.
\subsubsection{Connection to Natural Gradient}
\label{sec:natural_gradient}
\Cref{eq:delta_nablaG} closely resembles the natural gradient \citep{amari1998natural,kakade2001natural,peters2008natural} $\tilde{\nabla_{\theta}} \mathcal{G} = \left( F(\phi) \right)^{-1} \nabla_{\theta} \mathcal{G}$. In our case, instead of representing the FIM at a single point, $H^{(i)}$ instead represents an integral FIM along the path from the current occupancy parameter $\phi^{(i)}$ to the regularized target occupancy parameter $\hat{\phi}$. Intuitively, this averaging smoothens the FIM, making it less sensitive to local artifacts, while accounting for the upcoming geometry, vaguely resembling implicit gradient updates. Hence, our function-space updates can be viewed as performing a variant of natural gradient descent by stepping towards the steepest descent direction with respect to an \textit{iteration-dependent average Fisher metric} $H^{(i)}$. \Cref{lemma:natural_gradient} shows that this FIM preconditioned gradient is indeed just a convex interpolation of known quantities, allowing us to efficiently compute a geometry-aware update direction without having to estimate the FIM. \Cref{fig:dual_geometry} visualizes the geometry of $\theta$ and $\phi$ parameters, highlighting the functional update direction and how it closely resembles the natural gradient. \Cref{sec:direct_optimisation} further elaborates on the benefits of this update step, comparing $H^{(i)}$ to the best-possible preconditioner. 
\begin{figure}[t!]
    \centering
    \includegraphics[width=\columnwidth]{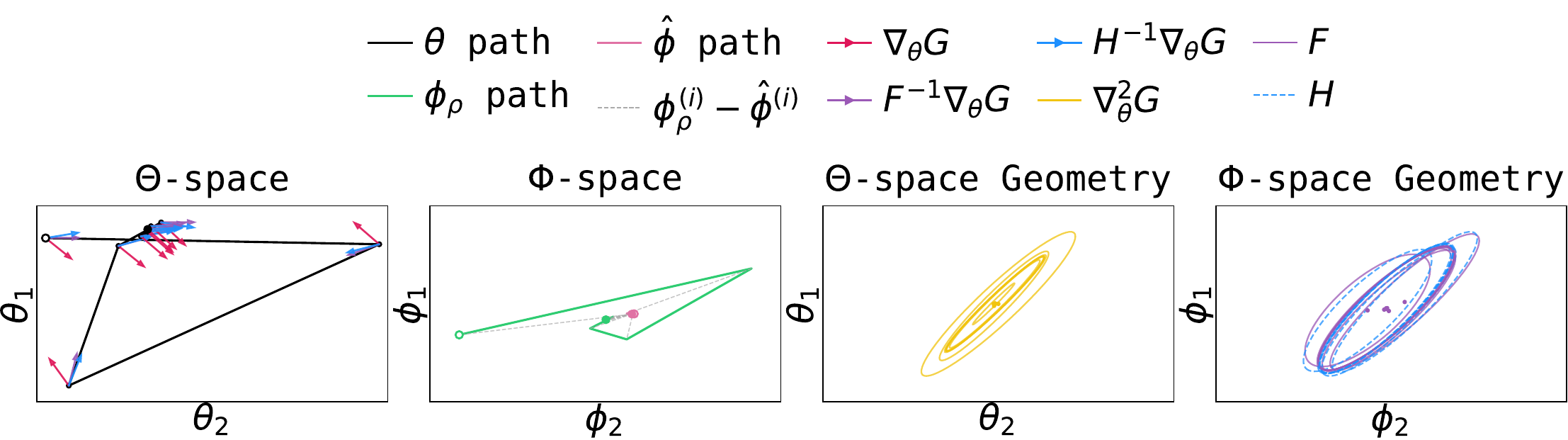}
    \caption{We visualise the vanilla gradient $\nabla_{\theta}\mathcal{G}$, the natural gradient $( F(\phi) )^{-1} \nabla_{\theta} \mathcal{G}$, and our integral FIM preconditioned gradient $\delta^{{(i)}} = \beta ( H^{(i)} )^{-1} \nabla_{\theta} \mathcal{G}$ in a gridworld IRL task. We see that the functional update direction is geometry-aware and finds a good change in reward parameters amidst anisotropic curvature. We also see that the geometry-aware update often aligns with the natural gradient, is intuitively more robust due to the integral over the e-geodesic, while also being much cheaper to compute. Finally, this experiment also confirms \Cref{fact:update_alignment} empirically. We provide experimental details in \Cref{appendix:experiment_details}.}
    \label{fig:dual_geometry}
\end{figure}
\subsubsection{Connection to Gradient Descent \& Direct Optimisation}
\label{sec:direct_optimisation}
In this section we expand on the connection between function-space reward updates and direct optimisation. We argue that the integral FIM preconditioned gradient is often better than the vanilla gradient by comparing the deviation of the obtained iterates from the optimal reward parameters. From \Cref{theorem:rewardMD,eq:delta_nablaG} we know that $\delta^{(i)} = \beta ( H^{(i)} )^{-1} \nabla_{\theta} \mathcal{G}$ is the reward update direction where $H^{(i)} \triangleq \int_0^1 F(\phi^{(i)}(\tau)) d\tau$ is an integral Fisher metric from the current occupancy parameter $\phi_{\rho^{(i)}}$ to an the regularized target occupancy parameter $\hat{\phi}$. The ``vanilla" gradient descent step just uses the identity matrix $I$ as a preconditioner.
\[
\begin{array}{ccc}
\textbf{Direct Optimisation} & \textbf{Functional Update} & \textbf{Vanilla Gradient Update} \\[2pt]
\theta^{\star} = \theta^{(i)} - \left(\bar{H}^{(i)}\right)^{-1} \nabla_{\theta}\mathcal{G}
&
\theta^{(i+1)} = \theta^{(i)} - \epsilon \beta \left(H^{(i)}\right)^{-1} \nabla_{\theta}\mathcal{G}
&
\theta^{(i+1)} = \theta^{(i)} - \alpha I \nabla_{\theta}\mathcal{G}
\end{array}
\]
We can then compare the next iterates $\theta^{(i+1)}$ from the functional update and vanilla gradient steps with $\theta^{\star}$, that is obtained by direct optimisation.
\begin{proposition}[Functional Update Contraction]
\label[proposition]{prop:functional_update_contraction}
With the optimal step size $\epsilon^{\star}$ (\Cref{eq:epsilon_star}), the functional update is first-order accurate with a contraction factor governed by the condition number of the preconditioned curvature matrix $\kappa((H^{(i)})^{-1}\bar H^{(i)})$.
\begin{align*}
    \left\lVert \theta^{(i+1)} - \theta^{\star} \right\rVert_{H^{(i)}} &\leq \frac{\kappa \left((H^{(i)})^{-1}\bar H^{(i)}\right) - 1}{\kappa \left((H^{(i)})^{-1}\bar H^{(i)}\right) + 1} \left\lVert\theta^{(i)} - \theta^\star\right\rVert_{H^{(i)}}
\end{align*}
\end{proposition}
\begin{proposition}[Vanilla Gradient Contraction]
\label[proposition]{prop:vanilla_gradient_contraction}
With the optimal step size $\alpha^{\star}$ (\Cref{eq:alpha_star}), the vanilla gradient step is first-order accurate with a contraction factor governed by the condition number of the local curvature matrix $\kappa(\bar{H}^{(i)})$. 
\begin{align*}
    \left\lVert \theta^{(i+1)} - \theta^{\star} \right\rVert_2 &\leq \frac{\kappa\!\left(\bar{H}^{(i)}\right) - 1}{\kappa\!\left(\bar{H}^{(i)}\right) + 1} \left\lVert \theta^{(i)} - \theta^{\star} \right\rVert_2
\end{align*}
\end{proposition}
Identity preconditioning ignores the geometry of the optimisation landscape and therefore inherits any anisotropy in $\bar{H}^{(i)}$. In contrast, $(H^{(i)})^{-1}$ in \Cref{prop:functional_update_contraction} rescales the update according to the average Fisher information along the e-geodesic to the target, so $(H^{(i)})^{-1}\bar{H}^{(i)}$ reflects only the residual curvature after this geometric preconditioning. We argue that most IRL problems have anisotropic geometry ($\kappa(\bar{H}^{(i)}) \gg 1$), and hence it is expected that functional updates often outperform vanilla gradient updates. Our conclusions align with the arguments for natural gradients made in \citet{kakade2001natural,peters2008natural} and are confirmed by the numerical analysis in \Cref{tab:angle_deviation,tab:contraction_rates,fig:dual_geometry}.
\section{Policy Search over Functional Reward Updates}
\label{sec:trirl_analysis}
We begin by recalling the standard algorithmic view of IRL. Primal-dual IRL methods alternate between updating the reward and optimising the policy induced by that reward, often running the policy optimisation step to convergence after each reward update. Our analysis so far has focused on the geometry of the reward update in occupancy space. We now turn to the policy side and ask: what policy update is induced when optimising the corresponding geometry-aware reward iterates? As before, we first introduce some assumptions and clarify some facts about the problem setting.
\subsection{Assumptions \& Remarks}
\begin{assumption}[Lipschitz Primal]
\label[assumption]{assumption:lipschitz_primal}
$\mathcal{J}(\policyOcc)$ is $L_{\mathcal{J}}$-Lipschitz Continuous w.r.t the $\ell1$ norm in the set $\varrho$, where $\varrho \subset \Delta_{\mathcal{S} \times \mathcal{A}}$ is the set of valid occupancies (satisfying the Bellman flow constraints).
\begin{align*}
    \left\lvert \mathcal{J}(\policyOcc) - \mathcal{J}(\rho_{\pi'}(\mathbf{s},\mathbf{a})) \right\rvert \leq L_{\mathcal{J}} \left\lVert \policyOcc - \rho_{\pi'}(\mathbf{s},\mathbf{a}) \right\rVert_1 \quad \forall \; \rho_{\pi}, \rho_{\pi'} \in \varrho
\end{align*}
\end{assumption}
\begin{fact}[Policy-Occupancy Relationship]
\label[fact]{fact:J_rho}
There is a unique relationship between a policy $\policy$ and its state-action occupancy $\policyOcc$, wherein $\policyOcc = \policy \policyDist$. Proof in \citep[Lemmas 3.2, 3.3]{ho2016generative} and \citep[Theorem 2]{syed2008apprenticeship}. 
\end{fact}
\begin{proposition}[Occupancy Mirror Map]
\label[proposition]{prop:rho_q_map}
$\nabla \Omega(\rho) : \Delta_{S \times A} \rightarrow \mathbb{R}^{S \times A}$ maps occupancies $\policyOcc \in \varrho$ to action-value functions $q(\mathbf{s}, \mathbf{a}) \in Q$ and, for a fixed state occupancy $\policyDist$, the gradient of its convex conjugate $\nabla \Omega^{\text{conj}}$ is the inverse map that maps action-value functions back to occupancies.
\end{proposition}
\begin{fact}[Bregman \& KL Div.]
\label[fact]{fact:bregmanKL}
Given a negative Shannon entropy mirror map $\Omega(\rho): \Delta_{\mathcal{S} \times \mathcal{A}} \to \mathbb{R}$ defined by $\Omega(\rho_{\pi}) \triangleq \sum_{\mathbf{s},\mathbf{a}} \policyOcc \log \policyOcc$, the Bregman divergence $D_{\Omega}(\policyOcc \left\lvert\right\rvert \policyOccIter{i})$ between state-action occupancies is equivalent to the KL divergence $\text{KL}(\policyOcc \left\lvert\right\rvert \policyOccIter{i})$ between the occupancies. The proof of this follows directly from the definition of the Bregman divergence.
\end{fact}
\begin{remark}
\Cref{prop:rho_q_map} re-confirms \Cref{fact:J_rho} from an MD point-of-view, by showing that a negative Shannon entropy mirror map in $\policyOcc$ space indeed still maps to the Q-space (and vise versa). This extends the discussion from \citet{geist2019theory} to occupancy spaces.
\end{remark}
\begin{fact}[Convexity]
\label[fact]{fact:convexity}
The optimisation objective \Cref{eq:IRLDM} is concave and differentiable in the occupancy (maximising it is a convex optimisation problem). $\Omega(\rho)$ is strictly convex and differentiable in $\varrho$ (Lemma~3.1 of \cite{ho2016generative}).
\end{fact}
\subsection{Policy Update Analysis}
We are now ready to analyse max-ent RL updates on geometry-aware reward updates. Below, we state our second main theorem, connecting our policy update to Entropic Mirror Descent.
\begin{theorem}[Policy-MD]
\label{theorem:policyMD}
A mirror descent update on \Cref{eq:IRLDM}, under negative Shannon entropy regularization is equivalent to the following max-ent policy update
\begin{align*}
\underset{\policy}{\arg \max} &\; \mathbb{E}_{\policyDist} \left[ H(\policy) + \mathbb{E}_\pi \left[ \left( \mathcal{U}_{\epsilon}^{\rho^{(i)}} \right) r^{(i)} \right] \right] - \mathcal{K}\\
\text{\textit{where,}} &\quad \mathcal{K} \leq \frac{\gamma}{1 - \gamma} \mathbb{E}_{\policyDist}  \left[ \text{KL} \left( \policy \left\lvert\right\rvert \policyIter{i} \right) \right].
\end{align*}
\end{theorem}
Hence, the policy iterates corresponding to the reward iterates perform approximate entropic mirror descent, where the error $\mathcal{K}$ is bounded by the reverse KL divergence to the previous iterate. In the special case of KL-constrained max-ent RL \citep{peters2010relative,schulman2015trust}, this deviation is naturally bounded by the trust-region constraint already enforced during policy optimisation. Trust-region max-ent RL on occupancy-updated rewards is therefore entropic MD, up to the same KL bound used in policy optimisation. This ties policy and reward iterates together: \coloremph{the policy update performs MD in occupancy space, while the reward update performs MD in the geometry induced by reverse-KL distribution matching}.
\section{Related Work}
\label{sec:prior_work}
Before concluding with a convergence analysis in \Cref{sec:convergence_analysis}, in this section, we briefly discuss related work in natural gradients and mirror descent in inverse reinforcement learning. Geometric optimisation in IRL has been studied on both the reward-side and policy-side subproblems. Prior work can be broadly classified into two categories: (i) methods that update the policy with natural policy gradients or mirror descent while updating the reward with vanilla gradient steps, and (ii) methods that update the reward function itself with mirror descent. In particular, we note that the natural gradient of the IRL dual problem with respect to reward parameters is largely unexplored.
\paragraph{Natural Policy Gradients (NPG) \& Policy MD in IRL.} \cite{zhang2020generative} establish global convergence guarantees for neural-network-based GAIL using alternating NPG updates and projected reward-gradient updates. \cite{guan2021will} extend the convergence analysis to several policy optimization methods, including NPG-GAIL and TRPO-GAIL. \cite{liu2022learning} study adversarial policy imitation with linear function approximation, providing guarantees for both online and offline learning. \cite{boularias2011relative} minimise the relative entropy between learned and baseline-policy trajectory distributions, but update rewards via importance-sampled dual subgradients. \cite{chen2026provably} combine model-based policy MD with off-policy reward updates. These approaches account for the geometry of policy optimization, while the reward parameters are updated using ordinary gradients. Their geometry-aware updates are thus confined to the policy and do not extend to the natural gradient of the IRL dual with respect to reward parameters.
\paragraph{Reward Function Mirror Descent.} Another line of work applies MD directly to reward optimization. \cite{syed2007game} use multiplicative weights on reward-feature weights, resembling entropic MD \citep{beck2003mirror}, while \cite{shani2022online} combine policy-side and cost-side MD for online apprenticeship learning. \cite{han2022robust} introduce MD for robust reward estimation under uncertainty in density estimation, while \cite{espinosa2025efficient} study reward-side MD under model misspecification. \cite{junca2026apprenticeship} employ stochastic MD for a regularized convex-concave apprenticeship objective, and \cite{swamy2023inverse} use MD as an optimization component in an IRL framework with expert-state resets. These methods introduce proximal terms on the reward side, but do not derive the reward update from the natural gradient of the IRL dual. In contrast, the reward interpolation studied in this paper corresponds to MD under a time-varying Mahalanobis proximal term induced by the average Fisher information along the occupancy-space e-geodesic.

To our knowledge, \cite{neu2012apprenticeship} is the only prior IRL work constructing a natural gradient in reward-parameter space---though, they optimize a policy-matching objective rather than the IRL dual. \cite{neu2012apprenticeship} consider the mapping $h(\theta)=\pi_\theta$ from reward parameters to the corresponding optimal policy and define the induced metric $G_\theta=h'(\theta)^\intercal h'(\theta)$, resulting in the natural gradient $\widetilde{\nabla}_\theta \widetilde{J}=G_\theta^\dagger \nabla_\theta \widetilde{J}$. Computing this direction requires differentiating the mapping from reward parameters to the optimal policy, including derivatives of the optimal action-value function $Q_\theta^\star$, and explicitly pseudo-inverting $G_\theta$. In contrast, we directly precondition the gradient of the IRL dual using the average Fisher information along an occupancy-space e-geodesic. Further, our update admits the closed-form expression $\delta^{(i)} = \theta^{(i)} - \beta \log (\nicefrac{\rho_E}{\rho_{\pi^{(i)}}})$ and can therefore be computed without explicitly estimating or inverting the FIM.
\section{Convergence Analysis}
\label{sec:convergence_analysis}
Here, we combine our insights so far to derive average-iterate convergence rates for an IRL method that uses $\Phi$-space reward updates along with KL-constrained max-ent policy optimisation. Our convergence argument follows the standard two-player saddle-point theory \citep{farina2019optimistic,grand2021conic}. The policy and reward updates are both regret-minimising algorithms, and the saddle point residual at the average iterate is bounded by the sum of their individual regrets. We bound the policy player regret in \Cref{theorem:policyRegret}, control the change in the integral FIM in \Cref{lemma:H_bound}, bound the reward player regret in \Cref{theorem:rewardRegret}, and combine the three in \Cref{theorem:SPR_convergence}. As before, we introduce some assumptions below.
\begin{assumption}[Bounded Reward Domain]
\label[assumption]{assumption:bounded_reward_domain}
The domain of reward parameters is bounded in the Euclidean norm
\begin{align*}
    \left\lVert \theta - \theta^{\prime} \right\rVert_2 \leq B \quad \forall \quad \theta, \theta^{\prime} \in \Theta
\end{align*}
\end{assumption}
\begin{assumption}[Lipschitz FIM]
\label[assumption]{assumption:lipschitz_FIM}
The Fisher Information Matrix $F(\phi) = \text{Cov}_{\rho_{\phi}}[\psi(\mathbf{s}, \mathbf{a})]$ is $L_F$-Lipschitz in the Euclidean norm (operator norm) in the set $\Phi$
\begin{align*}
    \left\lVert F(\phi) - F(\phi') \right\rVert_2 \leq L_F \left\lVert \phi - \phi' \right\rVert_2 \quad \forall \; \phi, \phi' \in \Phi
\end{align*}
\end{assumption}
\begin{theorem}[Policy Player Regret]
\label{theorem:policyRegret}
Under trust region policy updates (\Cref{theorem:policyMD}), the regret of the policy player is upper bounded as
\begin{align*}
\max_{\policyOcc} \overset{T}{\underset{i=1}{\sum}} \mathcal{J}(\policyOcc) - \overset{T}{\underset{i=1}{\sum}} \mathcal{J}(\policyOccIter{i}) \leq \policyLR \log \left( \left\lvert \mathcal{S} \right\rvert \left\lvert \mathcal{A} \right\rvert \right) + \frac{T L_{\mathcal{J}}^2 \log 2}{2 \policyLR}
\end{align*}
\end{theorem}
\begin{lemma}[Bounded Change in Integral FIM]
\label[lemma]{lemma:H_bound}
The Euclidean norm of the difference between integral Fisher Information Matrices is upper bounded as follows
\begin{align*}
    &\left\lVert H^{(i+1)} - H^{(i)} \right\rVert_2 \leq \frac{L_F}{2} \left( \frac{\epsilon L_{\mathcal{G}}}{m} + \sqrt{\frac{2 \zeta}{m(1 - \gamma)}} \right) \\
\end{align*}
\end{lemma}
\begin{theorem}[Reward Player Regret]
\label{theorem:rewardRegret}
The regret of the reward player is upper bounded as
\begin{align*}
\overset{T}{\underset{i=1}{\sum}} \mathcal{G}(\rewardIter{i}) - &\min_{\theta} \overset{T}{\underset{i=1}{\sum}} \mathcal{G}(r_{\theta}) \leq \\
&\frac{M B^2}{2 \epsilon \beta} + \frac{T B^2 L_F}{4 \epsilon \beta} \left( \frac{\epsilon L_{\mathcal{G}}}{m} + \sqrt{\frac{2 \zeta}{m(1 - \gamma)}} \right) + \frac{\epsilon \beta L_{\mathcal{G}}^2 T}{2 m}
\end{align*}
\end{theorem}
We are now ready to state the convergence results for our method. 
\begin{theorem}[SPR Bound \& Convergence]
\label{theorem:SPR_convergence}
Given learning rates $\policyLR = \sqrt{T}$ and $\epsilon = \nicefrac{1}{\sqrt{T}}$, trust-region tolerance $\zeta = \nicefrac{1}{T^2}$, and the average iterates $\Bar{\rho}(\mathbf{s},\mathbf{a}) = \nicefrac{1}{T} \sum_{i=1}^T \rho_{\pi^{(i)}}(\mathbf{s},\mathbf{a})$ (with $\Bar{\pi}$ obtained from marginalization) and $\Bar{r} = \nicefrac{1}{T} \sum_{i=1}^T r^{(i)}(\mathbf{s},\mathbf{a})$, the saddle point residual $\text{SPR}_T \triangleq \mathcal{G}(\Bar{r}) - \mathcal{J}(\Bar{\rho}(\mathbf{s},\mathbf{a}))$ converges sublinearly up to a fixed residual level
\begin{align*}
    \text{SPR}_{T} \leq \frac{B^2 L_F L_{\mathcal{G}}}{4 \beta m} + \mathcal{O}\left(\frac{1}{\sqrt{T}}\right)
\end{align*}
\end{theorem}

\section{Discussion \& Limitations}
\label{sec:discussion}
Intuitively, a key insight of our work is that, in imitation learning and inverse reinforcement learning, the Fisher information around the optimal solution is approximately available through the expert demonstrations. This creates a unique opportunity to move beyond local natural gradient updates by accounting for anticipated changes in the information geometry. We study entropy-regularized reverse KL minimisation and show that the average Fisher information along the e-geodesic to an estimate of the optimal occupancy naturally preconditions the gradient update, and can be incorporated efficiently without additional computational overhead. This preconditioner resembles the ideal preconditioner, which also averages a geometry along a straight line, but does so by integrating the Hessian—which depends on the complex mapping from reward parameters to occupancies—along the geodesic between the current and optimal reward parameters, the latter of which is unknown. Based on these insights, we establish sublinear convergence and show that the policy iterates induced by the preconditioned reward updates perform approximate entropic mirror descent, with approximation errors controlled by policy KL-divergences. These results align with the empirical findings of~\citet{arenz2016optimal} and the recent work of \citet{diwan2026trust}, which showed that this theoretically motivated update achieves state-of-the-art results in imitation learning and inverse reinforcement learning. Our theory complements these methodological advances and suggests a broader principle: using the expert’s information geometry as a lever across other divergences and problem formulations in imitation learning and inverse reinforcement learning.

While our analysis provides key insights into reward-function updates in IRL, several limitations remain. We do not explicitly account for errors from function approximation or other artifacts such as data availability and density-ratio estimation. We also assume the existence of a finite feature vector in which occupancies and rewards are linear. Although \Cref{appendix:assumptions} gives a practical method for computing such features, this still introduces approximation error. We briefly discuss error propagation under misspecification in \Cref{appendix:error_propagation}. Finally, our reward-player regret bound (\Cref{theorem:rewardRegret}) relies on \Cref{assumption:bounded_reward_domain}, which may not always hold in practice. Future work could focus on minimising the SPR residual in \Cref{theorem:SPR_convergence} and extending the analysis to IRL under general $f$-divergences.

\newpage

\begin{ack}
This research has been partially supported by the German Federal Ministry of Research, Technology and Space (BMFTR) within the project RobInTime (Grant: 01IS25002B) and under the Robotics Institute Germany (RIG), and by Huawei Tech R\&D (U.K.) through a research funding agreement. This project has received funding from the European Union's Horizon Europe research and innovation programme (Grant agreement ID: 101298421). This project has been supported by a hardware donation by NVIDIA through the Academic Grant Program.
\end{ack}

\bibliographystyle{plainnat}
\bibliography{neurips_2026}
\crefalias{section}{appendix}
\crefalias{subsection}{appendix}
\crefalias{subsubsection}{appendix}
\newpage
\appendix
\section{Extended Discussion}
\subsection{Feature Learning}
\label{appendix:assumptions}
Below, we support \Cref{assumption:boltzmann_occupancies} by showing how to learn a suitable finite-dimensional feature map $\psi(\mathbf{s}, \mathbf{a})$ from samples via score matching. Our key insight here is that $\phi_{E}^\intercal \; \psi(\mathbf{s}, \mathbf{a}) \in \mathbb{R}$ and $\phi_{\rho}^\intercal \; \psi(\mathbf{s}, \mathbf{a}) \in \mathbb{R}$ are energy functions of exponential family distributions. We can hence learn $\psi_{w}(\mathbf{s}, \mathbf{a})$ as a neural network feature encoder with parameters $w$ and map these learnt features to a scalar using an energy network $e_{v}(\psi_{w}(\mathbf{s}, \mathbf{a}))$ with parameters $v$. The parameters $v,w$ can be learnt by minimizing the denoising score matching loss \citep{song2019generative}
\begin{align*}
    \mathcal{L}_{\text{DSM}} = \mathbb{E}_{\rho_{E},\rho_{\pi}} \left[ \mathbb{E}_{\mathbf{x}' \sim \mathcal{N} \left( \mathbf{x}, \sigma \right)} \left[\frac{\mathbf{x}' - \mathbf{x}}{\sigma^2} - \nabla_{\mathbf{x}} e_{v}\left(\psi_{w}(\mathbf{x}')) \vert \sigma \right) \right] \right]
\end{align*}
where, $\mathbf{x}=[\mathbf{s};\mathbf{a}]$ is a state-action pair, and $\sigma$ is the standard deviation of the Gaussian noise added to perturb samples. The same procedure has been previously used to learn energy-based reward functions \citep{diwan2025noise} and we repurpose it to learn feature functions $\psi_{w}(\mathbf{s}, \mathbf{a})$.
\begin{table}[t!]
    \centering
    \caption{Effect of feature misspecification on the functional update (lower is better).}
    \label{tab:misspecification}
    \small
    \begin{tabular}{c cc}
        \toprule
        $r$
        & Gradient Error
        & Deviation [deg] \\
        \midrule
        0.0 & 5.7130e-05 &  0.000 \\
        0.5 & 7.5878e-05 &  0.663 \\
        1.0 & 9.4467e-05 &  3.683 \\
        1.5 & 1.1278e-04 & 12.028 \\
        2.0 & 1.3083e-04 & 31.374 \\
        \bottomrule
    \end{tabular}
\end{table}
\subsection{Error Propagation under Misspecified Features}
\label{appendix:error_propagation}
In this section, we analyse the effect of violating \Cref{assumption:boltzmann_occupancies} on the function-space update in \Cref{lemma:natural_gradient}. We posit that under the feature learning scheme presented in \Cref{appendix:assumptions}, a violation of \Cref{assumption:boltzmann_occupancies} is manifested as a feature misspecification term in the occupancy measure exponential family. By definition, the log-density of an exponential-family (EF) distribution is an affine function of features. Consider that $\tilde{\rho}_{\pi^{(i)}}, \tilde{\rho}_{E}$ are the true, potentially non-EF occupancies.
\begin{align*}
\tilde{\rho}_{\pi^{(i)}} &\propto \exp \{ \phi_{\rho^{(i)}}^\intercal \psi(\mathbf{s},\mathbf{a}) + h_{\rho^{(i)}}(\mathbf{s},\mathbf{a}) \} \\
\tilde{\rho}_{E} &\propto \exp \{ \phi_{E}^\intercal \psi(\mathbf{s},\mathbf{a}) + h_{E}(\mathbf{s},\mathbf{a}) \}
\end{align*}
where $h_{\rho^{(i)}}, h_E$ are the components of the energy function that are out of the span of $\psi(\mathbf{s},\mathbf{a})$ such that no linear combination of features can represent them. Consider approximating the agent's true occupancy using maximum-likelihood estimation (MLE) to obtain the parameter $\phi_{\rho^{(i)}}$ and EF occupancy $\rho_{\pi^{(i)}}$. Let $\phi_E$ denotes the in-span coefficient of $\log \tilde{\rho}_E$ above. We note that this analysis is valid since MLE and the DSM loss of \Cref{appendix:assumptions} share the same population minimizer.
\begin{fact}[M-projection]
\label[fact]{fact:MLE_Mproj}
The estimated MLE natural parameters solve 
\begin{align*}
    \phi_{\rho^{(i)}} = \arg \min_{\Phi} \text{KL} (\tilde{\rho}_{\pi^{(i)}} | \rho_{\pi^{(i)}}) \qquad \phi_{E} = \arg \min_{\Phi} \text{KL} (\tilde{\rho}_{E} | \rho_{E}).
\end{align*}
\end{fact}
\begin{fact}[Moment Matching]
\label[fact]{fact:Mproj_moment_matching}
M-projection is equivalent to moment matching. Hence, given $\phi_{\rho^{(i)}}, \phi_{E}$ from \Cref{fact:MLE_Mproj},
\begin{align*}
    \mu(\phi_{\rho^{(i)}}) = \mathbb{E}_{\tilde{\rho}_{\pi^{(i)}}} [\psi(\mathbf{s},\mathbf{a})] \qquad \mu(\phi_{E}) = \mathbb{E}_{\tilde{\rho}_{E}} [\psi(\mathbf{s},\mathbf{a})].
\end{align*}
\end{fact}
From \Cref{sec:occupancy_space_rew_updates} we note that $\tilde{\rho}_{\hat{\pi}}$ is obtained algebraically from $\theta^{(i)}$ and $\tilde{\rho}_{E}$. Hence the misspeficiation component $h_E$ propagates into it.
\begin{align*}
\log \tilde{\rho}_{\hat{\pi}} = \log \tilde{\rho}_{E} - \frac{\theta^{(i) \intercal} \psi(\mathbf{s},\mathbf{a})}{\beta} = \underbrace{\left(\phi_E - \frac{\theta^{(i)}}{\beta}\right)^\intercal \psi(\mathbf{s},\mathbf{a})}_{\hat{\phi}^\intercal \psi(\mathbf{s},\mathbf{a})} + h_E(\mathbf{s},\mathbf{a}) + \text{const}
\end{align*}
\begin{remark}
We note that $\hat{\phi}$ is the natural parameter of the EF distribution $\rho_{\hat{\phi}}$, not of the true target $\tilde{\rho}_{\hat{\pi}}$. Consequently $\mathbb{E}_{\tilde{\rho}_{\hat{\pi}}}[\psi] \neq \mu(\hat{\phi})$, i.e. the occupancy error manifests as an erroneous target.
\end{remark}
\begin{theorem}[Erroneous Functional Update]
\label{theorem:erroneous_functional_update}
The target error propagates into the functional update through an additive residual term
\begin{align*}
    \tilde{\delta}^{(i)} &= \beta (H^{(i)})^{-1} \left(\nabla_{\theta} \mathcal{G} + \mathbb{E}_{\tilde{\rho}_{\hat{\pi}}}[\psi(\mathbf{s},\mathbf{a})] - \mu(\hat{\phi})\right)
\end{align*}
\end{theorem}
\begin{proof}
\begin{align*}
    \text{\textit{From \cite{arenz2016optimal},}} &\text{\textit{ we know that the vanilla gradient definition holds without \Cref{assumption:boltzmann_occupancies}}} \\
    \nabla_{\theta} \mathcal{G} &= \mathbb{E}_{\tilde{\rho}_{\pi^{(i)}}}[\psi(\mathbf{s},\mathbf{a})] - \mathbb{E}_{\tilde{\rho}_{\hat{\pi}}}[\psi(\mathbf{s},\mathbf{a})] \\
    &\text{\textit{Add and subtract }} \mu(\phi_{\rho^{(i)}}), \mu(\hat{\phi}) \text{\textit{ and apply \Cref{fact:Mproj_moment_matching}}} \\
    &= \underbrace{\left(\mathbb{E}_{\tilde{\rho}_{\pi^{(i)}}}[\psi(\mathbf{s},\mathbf{a})] - \mu(\phi_{\rho^{(i)}})\right)}_{=\, 0} + \left(\mu(\phi_{\rho^{(i)}}) - \mu(\hat{\phi})\right) + \left(\mu(\hat{\phi}) - \mathbb{E}_{\tilde{\rho}_{\hat{\pi}}}[\psi(\mathbf{s},\mathbf{a})]\right) \\
    \mu(\phi_{\rho^{(i)}}) - \mu(\hat{\phi}) &= \nabla_{\theta} \mathcal{G} + \left(\mathbb{E}_{\tilde{\rho}_{\hat{\pi}}}[\psi(\mathbf{s},\mathbf{a})] - \mu(\hat{\phi})\right) \\
    &\text{\textit{Apply \Cref{lemma:natural_gradient}}} \\
    \tilde{\delta}^{(i)} &= \beta (H^{(i)})^{-1} \left(\nabla_{\theta} \mathcal{G} + \mathbb{E}_{\tilde{\rho}_{\hat{\pi}}}[\psi(\mathbf{s},\mathbf{a})] - \mu(\hat{\phi})\right)
\end{align*}
\end{proof}
We validate this analysis empirically by measuring how a controlled violation of \Cref{assumption:boltzmann_occupancies} corrupts $\delta^{(i)}$. We misspecify the expert such that $h_E(\mathbf{s},\mathbf{a})$ is constructed to be uncorrelated with every feature under $\rho_E$, i.e. $\mathrm{Cov}_{\rho_E}[\psi, h_E] = 0$, and its scale is $c=\sqrt{\mathbb{E}_{\rho_E}[h_E^2]}$. This ensures that it is outside the span of $\psi(\mathbf{s},\mathbf{a})$. We vary the misspecification error through the ratio of out-of-span and in-span magnitudes $r=\nicefrac{c}{\sqrt{\phi_E^\intercal \mathrm{Cov}_{\rho_E}[\psi] \phi_E}}$. We report the norm of the mean gradient error $\mathbb{E}_{\tilde{\rho}_{\hat{\pi}}}[\psi] - \mu(\hat{\phi})$ and mean angle deviation between $\delta^{(i)}$ and $\tilde\delta^{(i)}$ in \Cref{tab:misspecification}.
\subsection{Additional Results and Experiment Details}
\label{appendix:experiment_details}
Here, we elaborate on the experimental setup used for \Cref{fig:dual_geometry}. We use a 25-state gridworld with 2-dimensional features computed as described in \Cref{appendix:assumptions}. For IRL, we use Trust Region Inverse RL (TRIRL) \citep{diwan2026trust}---a recently proposed IRL method based on function-space reward updates---and generate a sequence of 2D reward parameters $\theta^{(i)}$ and policies $\policyIter{i}$ \footnote{Our implementation uses Numpy and Jax \citep{jax2020github}, hyperparameters $\{ \epsilon=0.1, \beta=10.0, \text{optim}=\text{SGD (with different gradients)} \}$, and does not require any specialised compute.}. We compute the closed-form discounted occupancy $\policyOccIter{i}$ using the standard equations, given the transition probabilities. Using \Cref{assumption:boltzmann_occupancies}, we fit occupancy parameters $\phi_{\rho^{(i)}}$ with a least-squares fit to $\log \policyOccIter{i}$. The target occupancy $\hat{\phi}$ is computed as in \Cref{eq:phi_hat}. The vanilla gradient is computed by evaluating $\mu(\phi) = \mathbb{E}_{\policyOccIter{i}} \left[ \psi(\mathbf{s}, \mathbf{a}) \right]$ at the current occupancy parameters and subtracting the corresponding value at the target occupancy. Similarly, the FIM is computed as ${\text{Cov}}_{\policyOccIter{i}}[\psi(\mathbf{s}, \mathbf{a})]$. $\nabla_{\theta}^2 \mathcal{G}$ is computed using finite differences of the vanilla gradient. We note that this numerical evaluation introduces minor approximation errors in the least-squares fitting of occupancy parameters and in the computation of FIMs. This leads to a small mismatch between the true functional reward updates in TRIRL and our preconditioned gradient direction. However, this approximation error is mainly implementation-related; the actual FIM-preconditioned reward update is always computed as a simple convex interpolation, as described in \Cref{sec:reward_analysis}.

To further support our arguments, we supplement \Cref{fig:dual_geometry}'s qualitative results with a numerical comparison of each gradient’s alignment with the direct optimization direction (\Cref{tab:angle_deviation}). Functional and natural gradients are often comparable, and both significantly outperform vanilla gradients. Finally, we compare contraction rates and condition numbers of the vanilla gradient and functional updates in \Cref{tab:contraction_rates}. Following \Cref{corollary:norm_change} we also correct for the fact that the norm changes in the functional update's case by premultiplying the contraction coefficient with the worst-case scaling coefficient $c = \sqrt{\lambda_{\max}((H^{(i)})^{-1}H^{(i+1)})}$. We see that even accounting for the norm change, the functional updates usually have a much better contraction coefficient than vanilla gradient descent.
\begin{table}[t!]
    \centering
    \caption{Angle deviation to the direct optimisation direction in degrees (lower is better).}
    \label{tab:angle_deviation}
    \small
    \begin{tabular}{c ccc}
        \toprule
        $i$
        & Vanilla Gradient
        & Natural Gradient
        & Functional Update \\
        \midrule
         0 & 80.2214 & 40.8667 & \textbf{32.1145} \\
         1 & 74.0856 & \textbf{13.4102} & 17.7227 \\
         2 & 75.4394 & 26.8519 & \textbf{18.1859} \\
         3 & 77.6527 & 22.6938 & \textbf{22.3652} \\
         4 & 73.5483 & \textbf{30.1869} & 39.6788 \\
         5 & 77.6062 & \textbf{25.0769} & 27.0140 \\
         6 & 76.3850 & \textbf{28.8652} & 34.5159 \\
         7 & 77.1939 & \textbf{26.7476} & 30.3462 \\
         8 & 76.9984 & \textbf{27.7353} & 32.0264 \\
         9 & 76.9745 & \textbf{27.4710} & 31.7271 \\
        10 & 77.0729 & \textbf{27.4215} & 31.4756 \\
        \bottomrule
    \end{tabular}
\end{table}
\begin{table}[t!]
    \centering
    \caption{Comparison of contraction coefficients and condition numbers for vanilla gradient descent and functional updates in the experiment in \Cref{fig:dual_geometry} (lower is better).}
    \label{tab:contraction_rates}
    \small
    \begin{tabular}{c cc cc cc}
        \toprule
        & \multicolumn{2}{c}{Vanilla Gradient}
        & \multicolumn{4}{c}{Functional Update} \\
        \cmidrule(lr){2-3} \cmidrule(lr){4-7}
        $i$
        & $\kappa(\bar{H}^{(i)})$
        & Coeff.
        & $\kappa\left((H^{(i)})^{-1}\bar{H}^{(i)}\right)$
        & Coeff.
        & $c$
        & $c \times \mathrm{Coeff.}$ \\
        \midrule
         0 & 158.9480 & \textbf{0.9875} & \textbf{15.1281} & 0.8760 & 1.1716 & 1.0263 \\
         1 &  55.1415 & 0.9644 &  \textbf{2.8121} & 0.4754 & 1.0863 & \textbf{0.5164} \\
         2 &  62.0156 & 0.9683 &  \textbf{3.9072} & 0.5924 & 1.0167 & \textbf{0.6023} \\
         3 &  89.2020 & 0.9778 &  \textbf{6.1228} & 0.7192 & 1.0080 & \textbf{0.7250} \\
         4 &  85.4321 & 0.9769 &  \textbf{5.6004} & 0.6970 & 1.0061 & \textbf{0.7012} \\
         5 &  85.9005 & 0.9770 &  \textbf{5.7352} & 0.7031 & 1.0002 & \textbf{0.7032} \\
         6 &  86.9851 & 0.9773 &  \textbf{5.7762} & 0.7048 & 1.0010 & \textbf{0.7055} \\
         7 &  85.6898 & 0.9769 &  \textbf{5.6840} & 0.7008 & 1.0005 & \textbf{0.7011} \\
         8 &  86.7132 & 0.9772 &  \textbf{5.7697} & 0.7046 & 1.0008 & \textbf{0.7051} \\
         9 &  86.0594 & 0.9770 &  \textbf{5.7091} & 0.7019 & 1.0006 & \textbf{0.7023} \\
        10 &  86.4038 & 0.9771 &  \textbf{5.7443} & 0.7035 & 1.0003 & \textbf{0.7036} \\
        \bottomrule
    \end{tabular}
\end{table}

\section[Background Proofs]{\Cref{sec:background} Details}
\label{appendix:background_proofs}
Here, we list some fundamental properties of the optimization problem in \Cref{eq:IRLDM}. 
\begin{fact}[Entropy Concavity]
\label[fact]{fact:concave_entropy}
The expected entropy $\mathbb{E}_{\policyDist} \left[ H(\policy) \right]$ is concave in $\policyOcc$. Proof in \citep[Lemma 3.1, Appendix A.1]{ho2016generative,neu2017unified}.
\end{fact}
\begin{fact}[Strong Duality]
\label[fact]{fact:strong_duality}
This fact extends \Cref{fact:convexity}. The optimization problem in \Cref{eq:IRLDM} has strong duality. The concavity of $\mathcal{J}(\policyOcc)$ follows naturally from \Cref{fact:concave_entropy} and prior works like \citep{achiam2017constrained,neu2017unified,geist2019theory}. The entropy term is concave, the negative KL term is concave, and the Bellman flow constraints are linear (all in $\policyOcc$). Hence Slater’s conditions hold and strong duality applies. Finally from \Cref{fact:J_rho}, the policy is equivalently obtained by optimizing over $\policyOcc$ followed by marginalization. We hence, also note the existence of an optimal Lagrangian multiplier.
\end{fact}
\section[Reward Proofs]{\Cref{sec:occupancy_space_rew_updates} Proofs}
\label{appendix:reward_proofs}
\begin{proof}[\textbf{Proof of \Cref{prop:mahaBregman}}]
\begin{align*}
    \nabla_{\theta} \Omega^{(i)}(\theta) &= \frac{H^{(i)} \theta}{\beta} \\
    D_{\Omega^{(i)}}(\theta, \theta^{\prime}) &= \Omega^{(i)}(\theta) - \Omega^{(i)}(\theta^{\prime}) - \left< \nabla_{\theta^{\prime}} \Omega^{(i)}(\theta^{\prime}), \theta - \theta^{\prime} \right> \\
    &= \frac{\left< \theta , H^{(i)} \theta \right>}{2 \beta} - \frac{\left< \theta^{\prime} , H^{(i)} \theta^{\prime} \right>}{2 \beta} - \frac{\left(\theta^{\prime}\right)^{\intercal} H^{(i)} \left( \theta - \theta^{\prime} \right)}{\beta} \\
    &= \frac{\left< \theta , H^{(i)} \theta \right>}{2 \beta} - \frac{2 \left< \theta^{\prime} , H^{(i)} \theta \right>}{2 \beta} + \frac{\left< \theta^{\prime} , H^{(i)} \theta^{\prime} \right>}{2 \beta} \\
    &= \frac{\left( \theta - \theta^{\prime} \right)^\intercal H^{(i)} \left( \theta - \theta^{\prime} \right)}{2 \beta} \\
    &= \frac{1}{2 \beta} \left\lVert \theta - \theta^{\prime} \right\rVert_{H^{(i)}}^2
\end{align*}
\end{proof}
\begin{proof}[\textbf{Proof of \Cref{prop:vanillaGradient}}]
For any occupancy parameter $\phi$, since $\rho_{\phi}(\mathbf{s}, \mathbf{a}) = Z^{-1} \exp{\left( \phi^\intercal \; \psi(\mathbf{s}, \mathbf{a}) \right)}$, it follows that 
\begin{align*}
    \mu(\phi) &= \nabla_{\phi} A(\phi) \\
    \text{\textit{then }} \nabla_{\phi} \mu(\phi) &= \nabla_{\phi}^2 A(\phi) = \underset{\rho_{\phi}}{\text{Cov}}[\psi(\mathbf{s}, \mathbf{a})] = F(\phi) \\
    \text{\textit{where }} A(\phi) &= \log Z = \log \underset{\mathbf{s}, \mathbf{a}}{\sum} \exp{(\phi^\intercal \; \psi(\mathbf{s}, \mathbf{a}))} \text{\textit{ is the log partition function and }} F(\phi) \text{\textit{ is the Fisher Information Matrix}} \\
    \text{\textit{and }} \nabla_{\theta} \mathcal{G} &= \mathbb{E}_{\policyOccIter{i}} \left[ \psi(\mathbf{s}, \mathbf{a}) \right] - \mathbb{E}_{\rho_{\hat{\pi}}} \left[ \psi(\mathbf{s}, \mathbf{a}) \right] = \mu(\phi_{\rho^{(i)}}) - \mu(\hat{\phi})
\end{align*}
\end{proof}
\begin{proof}[\textbf{Proof of \Cref{lemma:natural_gradient}}]
Before beginning, we recall \Cref{assumption:boltzmann_occupancies,assumption:lipschitzness_and_boundedness,assumption:nonflat_curvature,prop:vanillaGradient} which are central to this proof. By the fundamental theorem of calculus
\begin{align}
    \mu(\phi_{\rho^{(i)}}) - \mu(\hat{\phi}) = \underset{\hat{\phi} \rightarrow \phi_{\rho^{(i)}}}{\int} \nabla_{\phi} \mu(\phi) d\phi \; &= \overset{1}{\underset{0}{\int}} \nabla_{\phi} \mu(\phi^{(i)}(\tau)) \left( \phi_{\rho^{(i)}} - \hat{\phi} \right) d\tau \nonumber \\
    &= \left( \overset{1}{\underset{0}{\int}} \nabla_{\phi} \mu(\hat{\phi} + \tau \left(\phi_{\rho^{(i)}} - \hat{\phi} \right)) d\tau \right) \left( \phi_{\rho^{(i)}} - \hat{\phi} \right) \nonumber \\
    &= \left( \overset{1}{\underset{0}{\int}} \underset{\rho_{\phi^{(i)} (\tau)}}{\text{Cov}}[\psi(\mathbf{s}, \mathbf{a})] d\tau \right) \left( \phi_{\rho^{(i)}} - \hat{\phi} \right) = \left( \overset{1}{\underset{0}{\int}} F(\phi^{(i)}(\tau)) d\tau \right) \left( \phi_{\rho^{(i)}} - \hat{\phi} \right) \nonumber \\
    \text{\textit{let }} H^{(i)} &\triangleq \overset{1}{\underset{0}{\int}} F(\phi^{(i)}(\tau)) d\tau \text{\textit{ and from \Cref{sec:occupancy_space_rew_updates}, }} \delta^{\prime} = \nicefrac{\delta^{{(i)}}}{\beta} = \phi_{\rho^{(i)}} - \hat{\phi}\nonumber \\
    \text{\textit{then }} \nabla_{\theta} \mathcal{G} &= H^{(i)} \delta^{\prime} \nonumber \\
    \delta^{\prime} &= \left( H^{(i)} \right)^{-1} \nabla_{\theta} \mathcal{G} \nonumber \\
    \delta^{{(i)}} &= \beta \left( H^{(i)} \right)^{-1} \nabla_{\theta} \mathcal{G} 
\end{align}
\end{proof}
\begin{proof}[\textbf{Proof of \Cref{theorem:rewardMD}}]
Consider a time-varying quadratic mirror map $\Omega^{(i)}(\theta) \triangleq \frac{\left< \theta , H^{(i)} \theta \right>}{2 \beta}$ as given in \Cref{prop:mahaBregman}. The mirror descent objective on $\mathcal{G}$ under $\Omega^{(i)}$ with step size $\omega$ is given by
\begin{align*}
    \theta^{(i+1)} &= \underset{\theta}{\arg \min} \left[ \left< \nabla_{\theta} \mathcal{G}(\theta^{(i)}), \theta \right> + \frac{1}{2 \omega \beta} \left\lVert \theta - \theta^{(i)} \right\rVert_{H^{(i)}}^2  \right] \\
    \nabla_{\theta} \left[ \left< \nabla_{\theta} \mathcal{G}(\theta^{(i)}), \theta \right> + \frac{1}{2 \omega \beta} \left\lVert \theta - \theta^{(i)} \right\rVert_{H^{(i)}}^2  \right] &=  \nabla_{\theta} \mathcal{G}(\theta^{(i)}) + \frac{\left(\theta - \theta^{(i)}\right)}{\omega \beta} H^{(i)} \\
    \text{\textit{the gradient of the objective is zero at }} &\theta^{(i+1)} \\
    \nabla_{\theta} \mathcal{G}(\theta^{(i)}) + \frac{\left(\theta^{(i+1)} - \theta^{(i)}\right)}{\omega \beta} H^{(i)} &= \left. \nabla_{\theta} \left[ \left< \nabla_{\theta} \mathcal{G}(\theta^{(i)}), \theta \right> + \frac{1}{2 \omega \beta} \left\lVert \theta - \theta^{(i)} \right\rVert_{H^{(i)}}^2  \right] \right \rvert_{\theta^{(i+1)}} = 0 \\
    \theta^{(i+1)} &= \theta^{(i)} - \omega \beta \left( H^{(i)} \right)^{-1} \nabla_{\theta} \mathcal{G}(\theta^{(i)}) \\
    \theta^{(i+1)} &= \theta^{(i)} - \epsilon \delta^{(i)} \quad\text{\textit{from \Cref{eq:delta_nablaG} given }} \omega = \epsilon
\end{align*}
\end{proof}
\begin{proof}[\textbf{Proof of \Cref{prop:direct-dual-opt}}]
Applying the fundamental theorem of calculus on $\nabla_{\theta} \mathcal{G}$ from $\theta^{\star}$ to $\theta^{(i)}$
\begin{align*}
\nabla_{\theta} \mathcal{G}(\theta^{(i)}) -  \cancelto{0}{\nabla_{\theta} \mathcal{G}(\theta^{\star})} &= \underset{\theta^{\star} \rightarrow \theta^{(i)}}{\int} \nabla_{\theta} \left( \nabla_{\theta} \mathcal{G}(\theta) \right) d\theta = \overset{1}{\underset{0}{\int}} \nabla_{\theta} \left( \nabla_{\theta} \mathcal{G}(\theta^{(i)}(\tau)) \right) \left( \theta^{(i)} - \theta^{\star} \right) d\tau \\
&= \left( \overset{1}{\underset{0}{\int}} \nabla_{\theta} \nabla_{\theta} \mathcal{G} \left(\theta^{\star} + \tau \left(\theta^{(i)} - \theta^{\star} \right) \right) d\tau \right) \left( \theta^{(i)} - \theta^{\star} \right) \\
&= \left( \overset{1}{\underset{0}{\int}} \nabla_{\theta}^2 \mathcal{G}(\theta^{(i)}(\tau)) d\tau \right) \left( \theta^{(i)} - \theta^{\star} \right) \\
\text{\textit{let }} & \bar{H}^{(i)} = \overset{1}{\underset{0}{\int}} \nabla_{\theta}^2 \mathcal{G}(\theta^{(i)}(\tau)) d\tau \text{\textit{ and assuming it's invertible }} \\
\theta^{\star} &= \theta^{(i)} - \left(\bar{H}^{(i)}\right)^{-1} \nabla_{\theta} \mathcal{G}
\end{align*}
\end{proof}
\begin{fact}
\label[fact]{fact:self-adjoint}
Here we briefly prove that $\left( H^{(i)} \right)^{-1} \bar{H}^{(i)}$ is self-adjoint under the $H^{(i)}$-norm. We start with the fact that $H^{(i)}$ and $\bar H^{(i)}$ are symmetric. Let $A = (H^{(i)})^{-1}\bar H^{(i)}$. $\langle A\mathbf{x}, \mathbf{y}\rangle_{H^{(i)}} = \left((H^{(i)})^{-1}\bar H^{(i)}\mathbf{x}\right)^\intercal H^{(i)} \mathbf{y} = \mathbf{x}^\intercal \bar H^{(i)} \mathbf{y} = \mathbf{x}^\intercal H^{(i)}\left((H^{(i)})^{-1}\bar H^{(i)}\mathbf{y}\right) = \langle \mathbf{x}, A\mathbf{y}\rangle_{H^{(i)}}$.
\end{fact}
\begin{proof}[\textbf{Proof of \Cref{prop:functional_update_contraction}}]
\begin{align}
    \text{\textit{From \Cref{prop:direct-dual-opt},}}& \text{\textit{ it can be seen that }} \nonumber \\
    \nabla_{\theta} \mathcal{G} &= \bar{H}^{(i)} \left( \theta^{(i)} - \theta^{\star} \right) \nonumber \\
    \text{\textit{From }} &\text{\textit{\Cref{lemma:natural_gradient,theorem:rewardMD}, }} \nonumber \\
    \theta^{(i+1)} &= \theta^{(i)} - \epsilon \beta \left(H^{(i)}\right)^{-1} \nabla_{\theta}\mathcal{G} \nonumber \\
    \text{\textit{Subtracting }} \theta^{\star} &\text{\textit{ from both sides, substituting }} \nabla_{\theta}\mathcal{G} \text{\textit{, and factoring out}} \nonumber \\
    \theta^{(i+1)} - \theta^{\star} &= \left( I -  \epsilon \beta \left(H^{(i)}\right)^{-1} \bar{H}^{(i)} \right) \left( \theta^{(i)} - \theta^{\star} \right) \nonumber \\
    \text{\textit{Taking norms in the }} &H^{(i)}\text{\textit{-norm }} \left\lVert x \right\rVert_{H^{(i)}} = \sqrt{x^\intercal H^{(i)} x} \nonumber \\
    \left\lVert \theta^{(i+1)} - \theta^\star \right\rVert_{H^{(i)}} &\leq \left\lVert I -  \epsilon \beta \left(H^{(i)}\right)^{-1} \bar{H}^{(i)} \right\rVert_{H^{(i)}} \left\lVert \theta^{(i)} - \theta^\star \right\rVert_{H^{(i)}} \nonumber \\
    \left( H^{(i)} \right)^{-1} \bar{H}^{(i)} &\text{\textit{ is self-adjoint under the }} H^{(i)}\text{\textit{-inner product (\Cref{fact:self-adjoint})}} \nonumber \\
    \text{\textit{Hence, }} \left( H^{(i)} \right)^{-1} \bar{H}^{(i)} &\text{\textit{ admits an eigendecomposition }} \left(H^{(i)}\right)^{-1}\bar H^{(i)} = Q\Lambda Q^{-1} \nonumber \\
    \text{where Q is orthogonal in the } H^{(i)}&\text{\textit{-inner product and }} \Lambda = \mathrm{diag}(\lambda_1, \dots, \lambda_d) \text{\textit{ has real positive entries. From this }} \nonumber \\
    \left\lVert I -  \epsilon \beta \left(H^{(i)}\right)^{-1} \bar{H}^{(i)} \right\rVert_{H^{(i)}} &= \left\lVert I - \epsilon\beta\Lambda \right\rVert_2 \nonumber \\
    \text{\textit{The operator norm of }} I - &\epsilon\beta\Lambda \text{\textit{ (diagonal) is the largest absolute entry}} \nonumber \\
    &=\max\left\{ \left\lvert 1 - \epsilon\beta\lambda_{\min} \left((H^{(i)})^{-1}\bar H^{(i)}\right) \right\rvert,\ \left\lvert 1 - \epsilon\beta\lambda_{\max}\left((H^{(i)})^{-1}\bar H^{(i)}\right) \right\rvert \right\} \nonumber \\
    \text{\textit{The best step size }} &\epsilon^{\star} \text{\textit{ balances both magnitudes}} \nonumber \\
    1 - \epsilon^\star\beta\lambda_{\min} \left(\left(H^{(i)}\right)^{-1}\bar H^{(i)}\right) &= -\left(1 - \epsilon^\star\beta\lambda_{\max} \left(\left(H^{(i)}\right)^{-1}\bar H^{(i)}\right)\right) \nonumber \\
    \text{\textit{Solving for }} &\epsilon^{\star} \text{\textit{ yields}} \nonumber \\
    \epsilon^\star &= \frac{2}{\beta\left(\lambda_{\min} \left(\left(H^{(i)}\right)^{-1}\bar H^{(i)}\right) + \lambda_{\max} \left( \left(H^{(i)}\right)^{-1}\bar H^{(i)}\right)\right)} \label{eq:epsilon_star} \\
    \left\lVert I - \epsilon^\star\beta\,\left(H^{(i)}\right)^{-1}\bar H^{(i)} \right\rVert_{H^{(i)}} &= \frac{\lambda_{\max} \left(\left(H^{(i)}\right)^{-1}\bar H^{(i)}\right) - \lambda_{\min} \left(\left(H^{(i)}\right)^{-1}\bar H^{(i)}\right)}{\lambda_{\max} \left(\left(H^{(i)}\right)^{-1}\bar H^{(i)}\right) + \lambda_{\min} \left(\left(H^{(i)}\right)^{-1}\bar H^{(i)}\right)} \nonumber \\
    \text{\textit{Condition number }} \kappa\left(\left(H^{(i)}\right)^{-1}\bar H^{(i)}\right) &= \nicefrac{\lambda_{\max}\left(\left(H^{(i)}\right)^{-1}\bar H^{(i)}\right)}{\lambda_{\min}\left(\left(H^{(i)}\right)^{-1}\bar H^{(i)}\right)} \text{\textit{. Hence,}} \nonumber \\
    \left\lVert I - \epsilon^\star\beta \left(H^{(i)}\right)^{-1} \bar H^{(i)} \right\rVert_{H^{(i)}} &= \frac{\kappa \left(\left(H^{(i)}\right)^{-1}\bar H^{(i)}\right) - 1}{\kappa \left(\left(H^{(i)}\right)^{-1}\bar H^{(i)}\right) + 1} \nonumber
\end{align}
Therefore, the best one-step contraction for the functional update is,
\begin{align*}
\left\lVert \theta^{(i+1)} - \theta^\star \right\rVert_{H^{(i)}} \leq \frac{\kappa \left(\left(H^{(i)}\right)^{-1}\bar H^{(i)}\right) - 1}{\kappa \left(\left(H^{(i)}\right)^{-1}\bar H^{(i)}\right) + 1} \left\lVert \theta^{(i)} - \theta^\star \right\rVert_{H^{(i)}}
\end{align*}
\end{proof}
\begin{corollary}[Extension of \Cref{prop:functional_update_contraction} to the Euclidean norm]
\label[corollary]{corollary:euclidean_contraction}
From \Cref{assumption:uniform_covariate_conditioning}, $\sqrt{m} \; \left\lVert \mathbf{x} \right\rVert_2 \leq \left\lVert \mathbf{x} \right\rVert_{H^{(i)}} \leq \sqrt{M} \; \left\lVert \mathbf{x} \right\rVert_2$ for all $\mathbf{x} \in X$. Applying the lower bound to $\theta^{(i+1)} - \theta^\star$ and the upper bound to $\theta^{(i)} - \theta^\star$,
\begin{align*}
    \sqrt{m} \left\lVert \theta^{(i+1)} - \theta^\star \right\rVert_2 \leq \left\lVert \theta^{(i+1)} - \theta^\star \right\rVert_{H^{(i)}} &\leq \frac{\kappa \left(\left(H^{(i)}\right)^{-1}\bar H^{(i)}\right) - 1}{\kappa \left(\left(H^{(i)}\right)^{-1}\bar H^{(i)}\right) + 1}\left\lVert \theta^{(i)} - \theta^\star \right\rVert_{H^{(i)}} \\
    &\leq \frac{\kappa \left(\left(H^{(i)}\right)^{-1}\bar H^{(i)}\right) - 1}{\kappa \left(\left(H^{(i)}\right)^{-1}\bar H^{(i)}\right) + 1} \sqrt{M} \; \left\lVert \theta^{(i)} - \theta^\star \right\rVert_2 \\
    \left\lVert \theta^{(i+1)} - \theta^\star \right\rVert_2 &\leq \sqrt{\frac{M}{m}} \; \frac{\kappa \left(\left(H^{(i)}\right)^{-1}\bar H^{(i)}\right) - 1}{\kappa \left(\left(H^{(i)}\right)^{-1}\bar H^{(i)}\right) + 1}\left\lVert \theta^{(i)} - \theta^\star\right\rVert_2
\end{align*}
\end{corollary}
\begin{corollary}[Accounting for the changing norm in \Cref{prop:functional_update_contraction}]
\label[corollary]{corollary:norm_change}
Notice that the norm in \Cref{prop:functional_update_contraction} is iteration-dependent. Below we quantify the worst-case scaling factor of the updated reward parameter's error due to the change in the norm between iterations.
\begin{align*}
\left\lVert \theta^{(i+1)} - \theta^\star \right\rVert_{H^{(i+1)}} \leq \sqrt{\lambda_{\text{max}}\left( \left( H^{(i)} \right)^{-1} H^{(i+1)} \right)} \frac{\kappa \left(\left(H^{(i)}\right)^{-1}\bar H^{(i)}\right) - 1}{\kappa \left(\left(H^{(i)}\right)^{-1}\bar H^{(i)}\right) + 1} \left\lVert \theta^{(i)} - \theta^\star \right\rVert_{H^{(i)}}
\end{align*}
\end{corollary}
\begin{proof}
For an arbitrary vector $\mathbf{x} \in X$, the worst-case scaling factor due to the norm change is $c= \max_{\mathbf{x}} \frac{\left\lVert \mathbf{x} \right\rVert_{H^{(i+1)}}^2}{\left\lVert \mathbf{x} \right\rVert_{H^{(i)}}^2}$. Such that $\mathbf{x}$, $\left\lVert \mathbf{x} \right\rVert_{H^{(i+1)}}^2 \leq c \; \left\lVert \mathbf{x} \right\rVert_{H^{(i)}}^2$
\begin{align*}
c &=\max_{\mathbf{x}} \frac{\mathbf{x}^\intercal H^{(i+1)} \mathbf{x}}{\mathbf{x}^\intercal H^{(i)} \mathbf{x}} \\
\text{\textit{Setting the}} &\text{\textit{ gradient of this to zero, }} \\
H^{(i+1)}\mathbf{x}(\mathbf{x}^\intercal H^{(i)}\mathbf{x}) &= H^{(i)}\mathbf{x}(\mathbf{x}^\intercal H^{(i+1)}\mathbf{x}) \\
\left( H^{(i)} \right)^{-1} H^{(i+1)} \mathbf{x} &= \frac{\mathbf{x}^\intercal H^{(i+1)} \mathbf{x}}{\mathbf{x}^\intercal H^{(i)} \mathbf{x}} \mathbf{x} \\
\text{\textit{Setting }} A = \left( H^{(i)} \right)^{-1} H^{(i+1)} &\text{\textit{ and }} \frac{\mathbf{x}^\intercal H^{(i+1)} \mathbf{x}}{\mathbf{x}^\intercal H^{(i)} \mathbf{x}} = \lambda \text{\textit{, we see that }} A \mathbf{x} = \lambda \mathbf{x}. \\
\lambda \text{\textit{ is an eigenvalue of }}& (H^{(i)})^{-1} H^{(i+1)} \\
\text{\textit{The worst-case scaling factor for the unsquared norm is hence }} &\sqrt{\lambda_{\text{max}}\left( \left( H^{(i)} \right)^{-1} H^{(i+1)} \right)}
\end{align*}
Setting $\mathbf{x} = \theta^{(i+1)} - \theta^\star$,
\begin{align*}
\left\lVert \theta^{(i+1)} - \theta^\star \right\rVert_{H^{(i+1)}} &\leq \sqrt{\lambda_{\text{max}}\left( \left( H^{(i)} \right)^{-1} H^{(i+1)} \right)} \left\lVert \theta^{(i+1)} - \theta^\star \right\rVert_{H^{(i)}} \\
&\leq \sqrt{\lambda_{\text{max}}\left( \left( H^{(i)} \right)^{-1} H^{(i+1)} \right)} \frac{\kappa \left(\left(H^{(i)}\right)^{-1}\bar H^{(i)}\right) - 1}{\kappa \left(\left(H^{(i)}\right)^{-1}\bar H^{(i)}\right) + 1} \left\lVert \theta^{(i)} - \theta^\star \right\rVert_{H^{(i)}}
\end{align*}
\end{proof}
\begin{proof}[\textbf{Proof of \Cref{prop:vanilla_gradient_contraction}}]
\begin{align}
    \text{\textit{From \Cref{sec:reward_analysis}}}& \text{\textit{, the vanilla gradient step is}} \nonumber \\
    \theta^{(i+1)} &= \theta^{(i)} - \alpha I \nabla_{\theta}\mathcal{G} \nonumber \\
    \text{\textit{We use the definition of }} \nabla_{\theta}& \mathcal{G} \text{\textit{ from the proof of \Cref{prop:functional_update_contraction}}} \nonumber \\
    \text{\textit{Subtracting }} \theta^{\star} \text{\textit{ from both sides}}& \text{\textit{, substituting }} \nabla_{\theta}\mathcal{G} \text{\textit{, and factoring out}} \nonumber \\
    \theta^{(i+1)} - \theta^{\star} &= \left( I -  \alpha \bar{H}^{(i)} \right) \left( \theta^{(i)} - \theta^{\star} \right) \nonumber \\
    &\text{\textit{Taking norms }} \nonumber \\
    \left\lVert \theta^{(i+1)} - \theta^{\star} \right\rVert_2 &\leq \left\lVert I -  \alpha \bar{H}^{(i)} \right\rVert_2 \left\lVert \theta^{(i)} - \theta^{\star} \right\rVert_2 \nonumber \\
    \text{\textit{Since }} \bar{H}^{(i)} \text{\textit{ is symmetric positive definite, }} &\text{\textit{it admits an eigendecomposition }} \bar{H}^{(i)} = Q \Lambda Q^\intercal \nonumber \\ 
    \text{\textit{ where }} Q \text{\textit{ is orthogonal, }} \Lambda = \mathrm{diag}(\lambda_1,\dots,&\lambda_d) \text{\textit{. From orthogonal invariance}} \nonumber \\
    \left\lVert I -  \alpha \bar{H}^{(i)} \right\rVert_2 &= \left\lVert I -  \alpha Q \Lambda Q^\intercal \right\rVert_2 \nonumber \\
    &= \left\lVert Q \left( I - \alpha \Lambda \right) Q^\intercal \right\rVert_2 \nonumber \\
    &= \left\lVert I - \alpha \Lambda \right\rVert_2 \nonumber \\
    \text{\textit{The operator norm of }} I - &\alpha \Lambda \text{\textit{ (diagonal) is the largest absolute entry}} \nonumber \\
    &= \max \left\{ \left\lvert 1 - \alpha \lambda_{\min}\!\left(\bar{H}^{(i)}\right) \right\rvert, \left\lvert 1 - \alpha \lambda_{\max}\!\left(\bar{H}^{(i)}\right) \right\rvert \right\} \nonumber \\
    \text{\textit{The best step size }} &\alpha^{\star} \text{\textit{ balances both magnitudes}} \nonumber \\
    1 - \alpha^{\star} \lambda_{\min}\!\left(\bar{H}^{(i)}\right) &= - \left( 1 - \alpha^{\star} \lambda_{\max}\!\left(\bar{H}^{(i)}\right) \right) \nonumber \\
    \text{\textit{Solving for }} &\alpha^{\star} \text{\textit{ yields}} \nonumber \\
    \alpha^{\star} &= \frac{2}{\lambda_{\min}\!\left(\bar{H}^{(i)}\right) + \lambda_{\max}\!\left(\bar{H}^{(i)}\right)} \label{eq:alpha_star} \\
    \left\lVert I -  \alpha^{\star} \bar{H}^{(i)} \right\rVert_2 &= \frac{\lambda_{\max}\!\left(\bar{H}^{(i)}\right) - \lambda_{\min}\!\left(\bar{H}^{(i)}\right)}{\lambda_{\max}\!\left(\bar{H}^{(i)}\right) + \lambda_{\min}\!\left(\bar{H}^{(i)}\right)} \nonumber \\
    \text{\textit{Condition number }} \kappa\!\left(\bar{H}^{(i)}\right) &= \frac{\lambda_{\max}\!\left(\bar{H}^{(i)}\right)}{\lambda_{\min}\!\left(\bar{H}^{(i)}\right)} \text{\textit{. Hence, }} \nonumber \\
    \left\lVert I -  \alpha^{\star} \bar{H}^{(i)} \right\rVert_2 &= \frac{\kappa\!\left(\bar{H}^{(i)}\right) - 1}{\kappa\!\left(\bar{H}^{(i)}\right) + 1} \nonumber
\end{align}
Therefore, the best one-step contraction under identity preconditioning is,
\begin{align*}
    \left\lVert \theta^{(i+1)} - \theta^{\star} \right\rVert_2 &\leq \frac{\kappa\!\left(\bar{H}^{(i)}\right) - 1}{\kappa\!\left(\bar{H}^{(i)}\right) + 1} \left\lVert \theta^{(i)} - \theta^{\star} \right\rVert_2
\end{align*}
\end{proof}
\subsection{An Alternate Proof of First-Order Contraction}
Here, we provide an alternate proof of first-order contraction that leverages additional assumptions to provide a contraction rate in the Euclidean norm for any arbitrary step-size.
\begin{assumption}[Lipschitz Dual Hessian]
\label[assumption]{assumption:lipschitz_hessian}
$\nabla^2_{\theta} \mathcal{G}$ is locally $L_{\text{hess}_\mathcal{G}}$-Lipschitz Continuous in the Euclidean norm (operator norm) in the set $\Theta$
\begin{align*}
    \left\lVert \nabla^2_{\theta} \mathcal{G}(\theta) - \nabla^2_{\theta} \mathcal{G}(\theta') \right\rVert_2 \leq L_{\text{hess}_\mathcal{G}} \left\lVert \theta - \theta' \right\rVert_2 \quad \forall \; \theta, \theta' \in \Theta
\end{align*}
\end{assumption}
\begin{assumption}[Lipschitz Occupancy Parameter]
\label[assumption]{assumption:lipschitz_occupancy_parameter}
$\phi_{\rho}$ is locally $L_{\phi}$-Lipschitz Continuous in the Euclidean norm in the set $\Theta$, where $\rho$ and $\rho'$ are the occupancies under $\theta$ and $\theta'$
\begin{align*}
    \left\lVert \phi_{\rho} - \phi_{\rho'} \right\rVert_2 \leq L_{\phi} \left\lVert \theta - \theta' \right\rVert_2 \quad \forall \; \theta, \theta' \in \Theta
\end{align*}
\end{assumption}
\begin{proposition}[First Order Accuracy]
\label[proposition]{prop:first-ord-accurate}
The functional update step is first order accurate w.r.t the direct optimisation step, where $L_{\text{hess}_\mathcal{G}}$, $L_{F}$, and $ L_{\phi}$ are introduced in \Cref{assumption:lipschitz_hessian,assumption:lipschitz_FIM,assumption:lipschitz_occupancy_parameter}
\begin{align*}
    &\left\lVert \theta^{(i+1)} - \theta^{\star} \right\rVert_2 \\
    & \leq \frac{\epsilon \beta}{m} \left[ \left \lVert \frac{1}{\epsilon \beta}F(\phi_{\rho^{\star}}) - \nabla^2_{\theta} \mathcal{G}(\theta^{\star}) \right\rVert_2 + \left( \frac{L_{\text{hess}_\mathcal{G}}}{2} + \frac{L_{F}}{\epsilon \beta} \max \{ L_{\phi} , \frac{1}{\beta} \} \right) \left\lVert \theta^{(i)} - \theta^{\star} \right\rVert_2 \right] \left\lVert \theta^{(i)} - \theta^{\star} \right\rVert_2
\end{align*}
\end{proposition}
\begin{proof}[Proof]
\begin{align*}
    \text{\textit{From \Cref{prop:direct-dual-opt},}}& \text{\textit{ it can be seen that }} \\
    \nabla_{\theta} \mathcal{G} &= \bar{H}^{(i)} \left( \theta^{(i)} - \theta^{\star} \right) \\
    \text{\textit{From }} &\text{\textit{\Cref{lemma:natural_gradient,theorem:rewardMD}, }} \\
    \theta^{(i+1)} &= \theta^{(i)} - \epsilon \beta \left(H^{(i)}\right)^{-1} \nabla_{\theta}\mathcal{G} \\
    \text{\textit{Subtracting }} \theta^{\star} &\text{\textit{ from both sides, substituting }} \nabla_{\theta}\mathcal{G} \text{\textit{, and factoring out}} \\
    \theta^{(i+1)} - \theta^{\star} &= \left( I -  \epsilon \beta \left(H^{(i)}\right)^{-1} \bar{H}^{(i)} \right) \left( \theta^{(i)} - \theta^{\star} \right) \\
    &= \epsilon \beta \left(H^{(i)}\right)^{-1} \left( \frac{1}{\epsilon \beta}H^{(i)} - \bar{H}^{(i)} \right) \left( \theta^{(i)} - \theta^{\star} \right) \\
    &\text{\textit{Taking norms }} \\
    \left\lVert \theta^{(i+1)} - \theta^{\star} \right\rVert_2 &\leq \epsilon \beta \left\lVert \left(H^{(i)}\right)^{-1} \right\rVert_2 \left\lVert \frac{1}{\epsilon \beta}H^{(i)} - \bar{H}^{(i)} \right\rVert_2 \left\lVert \theta^{(i)} - \theta^{\star} \right\rVert_2 \\
    \text{\textit{From }}& \text{\textit{\Cref{assumption:uniform_covariate_conditioning}, }} \left( H^{(i)} \right)^{-1} I \preceq \frac{I}{m} \\
    \left\lVert \theta^{(i+1)} - \theta^{\star} \right\rVert_2 &\leq \frac{\epsilon \beta}{m} \left\lVert \frac{1}{\epsilon \beta}H^{(i)} - \bar{H}^{(i)} \right\rVert_2 \left\lVert \theta^{(i)} - \theta^{\star} \right\rVert_2 \\
\end{align*}
We now further improve this by bounding $ \left\lVert \frac{1}{\epsilon \beta}H^{(i)} - \bar{H}^{(i)} \right\rVert_2$ using Lipschitzness assumptions (\Cref{assumption:lipschitz_FIM,assumption:lipschitz_hessian,assumption:lipschitz_occupancy_parameter}). Recall the definitions $H^{(i)} \triangleq \int_0^1 F(\phi^{(i)}(\tau)) d\tau$ and $\bar{H}^{(i)} = \int_0^1 \nabla^2_{\theta} \mathcal{G}(\theta^{(i)}(\tau)) d\tau$ from \Cref{lemma:natural_gradient,prop:direct-dual-opt}. From this, it is straightforward that $H^{(i)}$ tends towards $F(\phi_{\rho^{\star}})$ as $\phi_{\rho^{(i)}}$ tends to $\phi_{\rho^{\star}}$ and $\bar{H}^{(i)}$ tends towards $\nabla^2_{\theta} \mathcal{G}(\theta^{\star})$ as $\theta^{(i)}$ tends towards $\theta^{\star}$ (naturally, at optimality $\phi_{\rho^{\star}}$ and $\theta^{\star}$ are quantities at the same optimal solution).
\begin{align*}
    &\text{\textit{Add and subtract }} \nabla^2_{\theta} \mathcal{G}(\theta^{\star}) \text{\textit{ and }} \nicefrac{1}{\epsilon \beta}F(\phi_{\rho^{\star}}) \\
    \frac{1}{\epsilon \beta}H^{(i)} - \bar{H}^{(i)} &= \left( \frac{1}{\epsilon \beta}H^{(i)} - \frac{1}{\epsilon \beta}F(\phi_{\rho^{\star}}) \right) + \left( \frac{1}{\epsilon \beta}F(\phi_{\rho^{\star}}) - \nabla^2_{\theta} \mathcal{G}(\theta^{\star}) \right) + \left( \nabla^2_{\theta} \mathcal{G}(\theta^{\star}) - \bar{H}^{(i)} \right) \\
    &\text{\textit{Triangle inequality}} \\
    \left \lVert \frac{1}{\epsilon \beta}H^{(i)} - \bar{H}^{(i)} \right\rVert_2 &\leq \left \lVert \frac{1}{\epsilon \beta}H^{(i)} - \frac{1}{\epsilon \beta}F(\phi_{\rho^{\star}}) \right\rVert_2 + \left \lVert \frac{1}{\epsilon \beta}F(\phi_{\rho^{\star}}) - \nabla^2_{\theta} \mathcal{G}(\theta^{\star}) \right\rVert_2 + \left \lVert \nabla^2_{\theta} \mathcal{G}(\theta^{\star}) - \bar{H}^{(i)} \right\rVert_2 \\
    &\text{\textit{Expand line integrals and triangle inequality}} \\
    &\leq \frac{1}{\epsilon \beta} \overset{1}{\underset{0}{\int}} \left\lVert F(\phi^{(i)}(\tau)) - F(\phi_{\rho^{\star}}) \right\rVert_2 d\tau + \left \lVert \frac{1}{\epsilon \beta}F(\phi_{\rho^{\star}}) - \nabla^2_{\theta} \mathcal{G}(\theta^{\star}) \right\rVert_2 +\\& \overset{1}{\underset{0}{\int}} \left\lVert \nabla^2_{\theta} \mathcal{G}( \theta^{\star} + \tau(\theta^{(i)} - \theta^{\star})) - \nabla^2_{\theta} \mathcal{G}(\theta^{\star}) \right\rVert_2 d\tau  \\
    &\leq \frac{1}{\epsilon \beta} \overset{1}{\underset{0}{\int}} L_{F} \left\lVert \phi^{(i)}(\tau) - \phi_{\rho^{\star}} \right\rVert_2 d\tau + \left\lVert \frac{1}{\epsilon \beta}F(\phi_{\rho^{\star}}) - \nabla^2_{\theta} \mathcal{G}(\theta^{\star}) \right\rVert_2 + \overset{1}{\underset{0}{\int}} L_{\text{hess}_\mathcal{G}} \tau \left\lVert \theta^{(i)} - \theta^{\star} \right\rVert_2 d\tau \\
    &\text{\textit{Points along a line are at max as far from another point as the end points of the line}} \\
    &\leq \frac{1}{\epsilon \beta} L_{F} \max \{ \left\lVert \phi_{\rho^{(i)}} - \phi_{\rho^{\star}} \right\rVert_2 , \left\lVert \phi_{\hat{\rho}} - \phi_{\rho^{\star}} \right\rVert_2 \} + \left \lVert \frac{1}{\epsilon \beta}F(\phi_{\rho^{\star}}) - \nabla^2_{\theta} \mathcal{G}(\theta^{\star}) \right\rVert_2 + \frac{L_{\text{hess}_\mathcal{G}}}{2} \left\lVert \theta^{(i)} - \theta^{\star} \right\rVert_2 \\
    &\text{\textit{Lipschitzness assumption and using }} \hat{\phi} = \phi_E - \frac{\theta}{\beta}  \\
    &\leq \frac{L_{F}}{\epsilon \beta} \max \{ L_{\phi} , \frac{1}{\beta} \} \left\lVert \theta^{(i)} - \theta^{\star} \right\rVert_2 + \left \lVert \frac{1}{\epsilon \beta}F(\phi_{\rho^{\star}}) - \nabla^2_{\theta} \mathcal{G}(\theta^{\star}) \right\rVert_2 + \frac{L_{\text{hess}_\mathcal{G}}}{2} \left\lVert \theta^{(i)} - \theta^{\star} \right\rVert_2 \\
    &= \left \lVert \frac{1}{\epsilon \beta}F(\phi_{\rho^{\star}}) - \nabla^2_{\theta} \mathcal{G}(\theta^{\star}) \right\rVert_2 + \left( \frac{L_{\text{hess}_\mathcal{G}}}{2} + \frac{L_{F}}{\epsilon \beta} \max \{ L_{\phi} , \frac{1}{\beta} \} \right) \left\lVert \theta^{(i)} - \theta^{\star} \right\rVert_2
\end{align*}
Combining this bound with our earlier result
\begin{align*}
    \left\lVert \theta^{(i+1)} - \theta^{\star} \right\rVert_2 &\leq \frac{\epsilon \beta}{m} \left[ \left \lVert \frac{1}{\epsilon \beta}F(\phi_{\rho^{\star}}) - \nabla^2_{\theta} \mathcal{G}(\theta^{\star}) \right\rVert_2 + \left( \frac{L_{\text{hess}_\mathcal{G}}}{2} + \frac{L_{F}}{\epsilon \beta} \max \{ L_{\phi} , \frac{1}{\beta} \} \right) \left\lVert \theta^{(i)} - \theta^{\star} \right\rVert_2 \right] \left\lVert \theta^{(i)} - \theta^{\star} \right\rVert_2
\end{align*}
\end{proof}
We now slightly extend \Cref{prop:first-ord-accurate} under an additional assumption on the curvature at optimality. 
\begin{assumption}[Equal Curvature at Optimality]
\label[assumption]{assumption:local_phi_invariance}
The natural parameter of the occupancy is first-order insensitive to perturbations of the reward parameter, at the optimal reward parameter $\theta^{\star}$. Such that $\left. \nabla_{\theta}  \phi_{\rho} \right|_{\theta^{\star}} = 0$.
\end{assumption}
We show that under \Cref{assumption:local_phi_invariance}, $\nabla^2_{\theta} \mathcal{G(\theta^{\star})} = \frac{1}{\beta}F(\phi_{\rho^\star})$
\begin{align*}
    \text{\textit{From }} &\text{\textit{\Cref{prop:vanillaGradient}}} \\
    \nabla_{\theta} \mathcal{G} &= \mu(\phi_{\rho^{(i)}}) - \mu(\hat{\phi}) \\
    \text{\textit{Differentiate w.r.t }} & \theta \\
    \nabla_{\theta}^2 \mathcal{G} &= F(\phi_{\rho^{(i)}}) \nabla_{\theta} \phi_{\rho^{(i)}} - F(\hat{\phi}) \nabla_{\theta} \hat{\phi} \\
    \text{\textit{From \Cref{assumption:boltzmann_occupancies}, }} & \hat{\phi} = \phi_{E} - \left( \nicefrac{\theta^{(i)}}{\beta} \right) \\
    \text{\textit{Evaluate at }}& \theta^{\star} \\
    \nabla_{\theta}^2 \mathcal{G}(\theta^{\star}) &= F(\phi_{\rho^{\star}}) \left. \nabla_{\theta} \phi_{\rho} \right|_{\theta^{\star}} + \frac{1}{\beta} F(\phi_{\rho^\star}) \\
    \text{\textit{Under}} &\text{\textit{ \Cref{assumption:local_phi_invariance}}} \\
    &= \frac{1}{\beta} F(\phi_{\rho^\star})
\end{align*}
As $\theta^{(i)} \rightarrow \theta^{\star}$, $\left(H^{(i)}\right)^{-1} \Bar{H}^{(i)} \rightarrow \nicefrac{I}{\beta}$ and the optimal step size of \Cref{prop:functional_update_contraction} $\epsilon^{\star} \rightarrow 1$. Setting $\epsilon = 1$, the first term of \Cref{prop:first-ord-accurate} vanishes. Hence, 
\begin{align*}
    \left\lVert \theta^{(i+1)} - \theta^{\star} \right\rVert_2 &\leq \frac{\epsilon \beta}{m} \left( \frac{L_{\text{hess}_\mathcal{G}}}{2} + \frac{L_{F}}{\epsilon \beta} \max \{ L_{\phi} , \frac{1}{\beta} \} \right) \left\lVert \theta^{(i)} - \theta^{\star} \right\rVert_2^2
\end{align*}
which implies that the geometry-aware update is second order accurate w.r.t the direct optimization step.
\subsection{Connection to Pseudo-Gradient Optimization in Euclidean Geometry}
\label{sec:connection_to_pseudo_gradient}
Here we show that $\delta^{(i)}$ has upper bounded norm and lower bound inner product with the gradient of the dual, implying that the functional reward update can be interpreted as a pseudo-gradient method in Euclidean geometry,
\begin{lemma}[$ \delta^{(i)}$ Properties]
\label[lemma]{lemma:pseudograd}
Let $r^{(i+1)} = \left( \updateOp{i} \right) r^{(i)} =  r^{(i)} - \epsilon \delta^{(i)}$ be a reward update where the update direction $\delta^{(i)}$ is an ascent direction with respect to the dual of \Cref{eq:IRLDM}, then:
\begin{align*}
  \left\lVert \delta^{(i)} \right\rVert &\leq \frac{\beta}{m} \left\lVert \nabla_{\theta} \mathcal{G} \right\rVert \\
  \left< \nabla_{\theta} \mathcal{G} , \delta^{(i)} \right> &\geq \frac{\beta}{M} \left\lVert \nabla_{\theta} \mathcal{G} \right\rVert^2
\end{align*}
\end{lemma}
\begin{proof}
From \Cref{lemma:natural_gradient}, we know that $\delta^{(i)} = \beta \delta^{\prime} = \beta \left( H^{(i)} \right)^{-1} \nabla_{\theta} \mathcal{G}$. From \Cref{assumption:uniform_covariate_conditioning}, \\
\begin{align*}
    mI &\preceq H^{(i)} I \preceq MI \\
    \frac{I}{M} &\preceq \left( H^{(i)} \right)^{-1} I \preceq \frac{I}{m} \\\\
    \text{\textit{\textbf{lower bounded}}}& \text{\textit{\textbf{ inner product}}} \\
    \left< \nabla_{\theta} \mathcal{G} , \delta^{(i)} \right> &= \beta \left< \nabla_{\theta} \mathcal{G} , \delta^{\prime} \right> = \beta \; \left( \nabla_{\theta} \mathcal{G} \right)^\intercal \left( H^{(i)} \right)^{-1} \nabla_{\theta} \mathcal{G} \\
    &\geq \beta \; \left( \nabla_{\theta} \mathcal{G} \right)^\intercal \frac{I}{M} \nabla_{\theta} \mathcal{G} \\
    &\geq \frac{\beta}{M} \left\lVert \nabla_{\theta} \mathcal{G} \right\rVert^2 \\\\
    &\text{\textit{\textbf{upper bounded}}} \text{\textit{\textbf{ norm}}} \\
    \left\lVert \delta^{(i)} \right\rVert &= \beta \left\lVert \delta^{\prime} \right\rVert = \beta \left\lVert \left( H^{(i)} \right)^{-1} \nabla_{\theta} \mathcal{G} \right\rVert \leq \beta \left\lVert \left( H^{(i)} \right)^{-1} \right\rVert_2  \left\lVert \nabla_{\theta} \mathcal{G} \right\rVert \\
    \text{\textit{since }} & H^{(i)} \text{\textit{ is positive definite}} \left\lVert \left( H^{(i)} \right)^{-1} \right\rVert_2 = \lambda_{\text{max}} \left(\left( H^{(i)} \right)^{-1}\right) = \frac{1}{\lambda_{\text{min}} \left( H^{(i)} \right)} \\
    \text{\textit{where }} & \lambda_{\text{max}}(.) \text{\textit{ and }} \lambda_{\text{min}}(.) \text{\textit{ are functions that return the max and min eigenvalues}} \\
    \text{\textit{and }} & \lambda_{\text{min}} \left( H^{(i)} \right) \geq m \\
    \text{\textit{hence }} \left\lVert \delta^{(i)} \right\rVert &\leq \frac{\beta}{\lambda_{\text{min}} \left( H^{(i)} \right)} \left\lVert \nabla_{\theta} \mathcal{G} \right\rVert \leq \frac{\beta}{m} \left\lVert \nabla_{\theta} \mathcal{G} \right\rVert
\end{align*}
\end{proof}
An interesting corollary arising from \Cref{lemma:pseudograd} and \cite{bertsekas2000gradient} is that in Euclidean geometry, our update direction $\delta^{(i)}$ is a pseudo-gradient of the vanilla gradient $\nabla_{\theta} \mathcal{G}$. This further strengthens \Cref{fact:update_alignment} since under \Cref{assumption:uniform_covariate_conditioning}, the update alignment is now lower bounded by the norm of the vanilla gradient.
\section[TRIRL Proofs]{\Cref{sec:trirl_analysis} Proofs}
\begin{proof}[\textbf{Proof of \Cref{prop:rho_q_map}}]
For a state $\mathbf{s}$, $\Omega(\policyOcc) = \underset{\mathbf{a}}{\sum} \policyOcc \log \policyOcc$. Its convex conjugate $\Omega^{\text{conj}}$ is defined as
\begin{align*}
    \Omega^{\text{conj}}(q(\mathbf{s}, \mathbf{a})) &= \underset{\policyOcc}{\sup}\left\{ \Big< q(\mathbf{s}, \mathbf{a}), \policyOcc \Big> - \Omega(\policyOcc) \right\} \\
    &= \underset{\policyOcc}{\sup}\left\{ \underset{\mathbf{a}}{\sum} q(\mathbf{s}, \mathbf{a}) \policyOcc - \underset{\mathbf{a}}{\sum} \policyOcc \log \policyOcc \right\} \\
    \policyOccOpt &= \underset{\policyOcc}{\arg \max} \underset{\mathbf{a}}{\sum} q(\mathbf{s}, \mathbf{a}) \policyOcc - \underset{\mathbf{a}}{\sum} \policyOcc \log \policyOcc \\
\end{align*}
$\policyOccOpt$ is optimal when the gradient of the inner term's Lagrangian is zero. 
\begin{align*}
    \nabla_{\policyOcc} \left[ \underset{\mathbf{a}}{\sum} q(\mathbf{s}, \mathbf{a}) \policyOcc - \underset{\mathbf{a}}{\sum} \policyOcc \log \policyOcc - \lambda(\mathbf{s}) \left( \underset{\mathbf{a}}{\sum} \policyOcc - \policyDistOpt \right) \right] &= 0 \\
    q(\mathbf{s}, \mathbf{a}) -1 -\log \policyOccOpt - \lambda(\mathbf{s}) &= 0 \\
    \policyOccOpt &= \exp{(q(\mathbf{s}, \mathbf{a}))} \exp{(-1 - \lambda(\mathbf{s}))} \\
    \text{\textit{Equivalently, }} \policyOpt &= \frac{\exp{(q(\mathbf{s}, \mathbf{a}))} \exp{(-1 - \lambda(\mathbf{s}))}}{\policyDistOpt} \\
    \text{\textit{Since}} \underset{\mathbf{a}}{\sum} \policyOpt = 1 \text{\textit{ we get }} \exp{(-1 - \lambda(\mathbf{s}))} = \nicefrac{\policyDistOpt}{\underset{\mathbf{a}}{\sum}\exp{(q(\mathbf{s}, \mathbf{a}))}} &= \nicefrac{\policyDistOpt}{Z_{\mathbf{s}}} \\
   \policyOccOpt &= \frac{\policyDistOpt \exp{(q(\mathbf{s}, \mathbf{a}))}}{Z_{\mathbf{s}}}  \\
   \policyOpt &= \frac{\exp{(q(\mathbf{s}, \mathbf{a}))}}{Z_{\mathbf{s}}}
\end{align*}
$\Omega^{\text{conj}}(q^{\star}(\mathbf{s}, \mathbf{a}))$ corresponds to $\policyOccOpt$.
\begin{align*}
    \Omega^{\text{conj}}(q^{\star}(\mathbf{s}, \mathbf{a})) &= \underset{\mathbf{a}}{\sum} q^{\star}(\mathbf{s}, \mathbf{a}) \policyOccOpt - \underset{\mathbf{a}}{\sum} \policyOccOpt \log \policyOccOpt \\
    &= \underset{\mathbf{a}}{\sum} \policyOccOpt \left[ q^{\star}(\mathbf{s}, \mathbf{a}) - \log \policyOccOpt \right] \\
    &= \left( \log Z_{\mathbf{s}} - \log \policyDistOpt \right) \frac{\policyDistOpt}{Z_{\mathbf{s}}} \underset{\mathbf{a}}{\sum} \exp{(q^{\star}(\mathbf{s}, \mathbf{a}))} \\
    &= \policyDistOpt \log Z_{\mathbf{s}} - \policyDistOpt \log \policyDistOpt
\end{align*}
Finally $\nabla_{q(\mathbf{s}, \mathbf{a})} \Omega^{\text{conj}}(q^{\star}(\mathbf{s}, \mathbf{a})) = \frac{\policyDistOpt \exp{(q^{\star}(\mathbf{s}, \mathbf{a}))}}{Z_{\mathbf{s}}} = \policyOccOpt$
\end{proof}
\begin{proposition}[Reward - Policy Relationship]
\label[proposition]{prop:r_logpi}
\begin{align*}
\log \frac{\policyOccIter{i}}{\policyDistIter{i}} = \log \policyIter{i} = \rewardIter{i} + \gamma\, \mathbb{E}_{\mathbf{s}' \sim \mathcal{P}(\cdot|\mathbf{s},\mathbf{a})} \left[
  V_{\pi^{(i)}}(\mathbf{s}') \right] - V_{\pi^{(i)}}(\mathbf{s})
\end{align*}
\end{proposition}
\begin{proof}[\textbf{Proof}]
This follows from \Cref{fact:J_rho} and the fact that the optimal policy for $\rewardIter{i}$ is Boltzmann under Shannon entropy regularization.
\end{proof}
\begin{lemma}[Gradient of $\mathcal{J}$]
\label[lemma]{lemma:gradient}
At $\policyOccIter{i}$, the gradient of $\mathcal{J}(\policyOcc)$ w.r.t. $\policyOcc$ equals the update direction $-\delta^{(i)}$ up to a constant and a potential-shaping term. Hence, it
induces the same optimal policy
\begin{align*}
  \nabla_{\policyOcc} \mathcal{J}(\policyOccIter{i}) = -\delta^{(i)} - \beta - \left( \gamma\, \mathbb{E}_{\mathbf{s}' \sim \mathcal{P}(\cdot|\mathbf{s},\mathbf{a})} \left[
V_{\pi^{(i)}}(\mathbf{s}') \right] - V_{\pi^{(i)}}(\mathbf{s}) \right)
\end{align*}
\end{lemma}
\begin{proof}[\textbf{Proof of \Cref{lemma:gradient}}]
\begin{align*}
    \mathcal{J}(\policyOcc) &= \mathbb{E}_{\policyDist} \left[ H(\policy) \right] - \beta \mathbb{E}_{\policyOcc} \left[ \log \frac{\policyOcc}{\expOcc} \right] \\
    \nabla_{\policyOcc} \mathcal{J}(\policyOcc) &= \nabla_{\policyOcc} \left[ \left( -\underset{\mathbf{s}, \mathbf{a}}{\sum} \policyOcc \log \frac{\policyOcc}{\policyDist} \right) - \beta \left( \underset{\mathbf{s}, \mathbf{a}}{\sum} \policyOcc \log \frac{\policyOcc}{\expOcc} \right) \right] \\
    &= - \log \frac{\policyOcc}{\policyDist} - \beta \left( \log \frac{\policyOcc}{\expOcc} + 1 \right) \\
    &\text{\textit{At iteration }} i \\
    &= - \left( \log \frac{\policyOccIter{i}}{\policyDistIter{i}} - \beta \log \frac{\expOcc}{\policyOccIter{i}} +\beta \right) \\
    &= - \left( \rewardIter{i} - \beta \log \frac{\expOcc}{\policyOccIter{i}} + \gamma\, \mathbb{E}_{\mathbf{s}' \sim \mathcal{P}(\cdot|\mathbf{s},\mathbf{a})} \left[ V_{\pi^{(i)}}(\mathbf{s}') \right] - V_{\pi^{(i)}}(\mathbf{s}) +\beta \right) \\
    \text{\textit{From \Cref{prop:r_logpi} and since }} &\delta^{(i)} = \rewardIter{i} - \beta \log \frac{\expOcc}{\policyOccIter{i}} \\
    &= -\delta^{(i)} - \beta - \left( \gamma\, \mathbb{E}_{\mathbf{s}' \sim \mathcal{P}(\cdot|\mathbf{s},\mathbf{a})} \left[ V_{\pi^{(i)}}(\mathbf{s}') \right] - V_{\pi^{(i)}}(\mathbf{s}) \right) \\
    \text{\textit{By policy invariance under potential}} &\text{\textit{ shaping \citep{ng1999policy}, }} \nabla_{\policyOcc} \mathcal{J} \text{\textit{ and }} -\delta^{(i)} \text{\textit{ induce the same optimal policy.}}
\end{align*}
\end{proof}
\Cref{prop:r_logpi,lemma:gradient} highlight that the optimizing the reward is equivalent to optimizing the log-policy, and optimizing $\nabla_{\policyOcc} \mathcal{J}(\policyOcc)$ is equivalent to optimizing with our update direction $\delta^{(i)}$.
\begin{proof}[\textbf{Proof of \Cref{theorem:policyMD}}]
Before beginning the proof, we list some intermediate results below:
\begin{align}
    \ThmStepLabel{theorem:policyMD-A}
    \tag{\thethmstep}
    \text{KL} \left(  \policyOcc \left\lvert \right\rvert \policyOccIter{i} \right) &= \text{KL}\left( \policyDist \left\lvert\right\rvert \policyDistIter{i} \right) + \underset{\mathbf{s}}{\sum} \policyDist \text{KL}\left( \policy \left\lvert\right\rvert \policyIter{i} \right) \\
    \ThmStepLabel{theorem:policyMD-B}
    \tag{\thethmstep}
    - \text{KL}\left( \policy \left\lvert\right\rvert \policyIter{i} \right) &= H(\policy) + \mathbb{E}_{\pi} \left[ \log \policyIter{i} \right] 
\end{align}
\Cref{fact:J_rho} implies that optimizing $\mathcal{J}(\policy)$ under the policy $\policy$ is equivalent to optimizing $\mathcal{J}(\policyOcc)$ under $\policyOcc$ and subsequently obtaining $\policy$ via marginalization. The mirror descent objective on $\mathcal{J}(\policyOcc)$ with step size $\omega$ is given by
\begin{align*}
    \policyOccIter{i+1}& \\ 
    &= \underset{\policyOcc}{\arg\max} \; \left[ \; \Big< \nabla_{\policyOcc} \mathcal{J}(\policyOccIter{i}), \policyOcc \Big> - \omega \text{KL} \left(  \policyOcc \left\lvert \right\rvert \policyOccIter{i} \right) \;\right] \\
    &= \underset{\policyOcc}{\arg\max} \; \left[ \underset{\mathbf{s}, \mathbf{a}}{\sum} \policyOcc \nabla_{\policyOcc} \mathcal{J}(\policyOccIter{i}) - \omega \text{KL} \left(  \policyOcc \left\lvert \right\rvert \policyOccIter{i} \right)  \right] \\
    &\text{\textit{From \Cref{theorem:policyMD-A}}} \\
    &= \underset{\policyOcc}{\arg\max} \; \Bigg[ \underset{\mathbf{s}, \mathbf{a}}{\sum} \policyDist \policy \nabla_{\policyOcc} \mathcal{J}(\policyOccIter{i}) \\&- \omega \left( \text{KL}\left( \policyDist \left\lvert\right\rvert \policyDistIter{i} \right) + \underset{\mathbf{s}}{\sum} \policyDist \text{KL}\left( \policy \left\lvert\right\rvert \policyIter{i} \right) \right) \Bigg] \\
    &\text{\textit{From \Cref{theorem:policyMD-B}}} \\
    &= \underset{\policyOcc}{\arg\max} \; \Bigg[ \underset{\mathbf{s}, \mathbf{a}}{\sum} \policyDist \policy \nabla_{\policyOcc} \mathcal{J}(\policyOccIter{i}) \\& +\omega \underset{\mathbf{s}}{\sum}{\policyDist}\left( H(\policy) + \mathbb{E}_{\pi}\left[ \log \policyIter{i} \right] \right) - \omega \text{KL}\left( \policyDist \left\lvert\right\rvert \policyDistIter{i} \right) \Bigg] \\
    &\text{\textit{Let }} \mathcal{K} = \text{KL}\left( \policyDist \left\lvert\right\rvert \policyDistIter{i} \right) \\
    &= \underset{\policyOcc}{\arg\max} \; \Bigg[ \underset{\mathbf{s}}{\sum} \policyDist \underset{\mathbf{a}}{\sum} \policy \nabla_{\policyOcc} \mathcal{J}(\policyOccIter{i}) \\& +\omega \underset{\mathbf{s}}{\sum}{\policyDist}\left( H(\policy) + \mathbb{E}_{\pi}\left[ \log \policyIter{i} \right] \right) -\omega \mathcal{K} \Bigg] \\
    &= \underset{\policyOcc}{\arg\max} \; \Bigg[ \underset{\mathbf{s}}{\sum} \policyDist \left[ \omega H(\policy) + \underset{\mathbf{a}}{\sum} \policy \nabla_{\policyOcc} \mathcal{J}(\policyOccIter{i}) + \omega \mathbb{E}_{\pi}\left[ \log \policyIter{i} \right] \right] -\omega \mathcal{K} \Bigg] \\
    &= \underset{\policyOcc}{\arg\max} \; \mathbb{E}_{\policyDist} \left[ H(\policy) + \mathbb{E}_{\pi}\left[ \log \policyIter{i} + \frac{1}{\omega} \nabla_{\policyOcc} \mathcal{J}(\policyOccIter{i}) \right]  \right] -\mathcal{K} \\
    &= \underset{\policyOcc}{\arg\max} \; \mathbb{E}_{\policyDist} \Bigg[ H(\policy) + \mathbb{E}_{\pi} \Big[ \rewardIter{i} + \gamma\, \mathbb{E}_{\mathbf{s}' \sim \mathcal{P}(\cdot|\mathbf{s},\mathbf{a})} \left[ V_{\pi^{(i)}}(\mathbf{s}') \right] - V_{\pi^{(i)}}(\mathbf{s}) \\& + \frac{1}{\omega} \nabla_{\policyOcc} \mathcal{J}(\policyOccIter{i}) \Big]  \Bigg] -\mathcal{K} \\
    &\text{\textit{Apply \Cref{prop:r_logpi,lemma:gradient}, and policy invariance under potential shaping \citep{ng1999policy}}} \\
    &= \underset{\policyOcc}{\arg\max} \; \mathbb{E}_{\policyDist} \left[ H(\policy) + \mathbb{E}_{\pi}\left[ \rewardIter{i} -\frac{1}{\omega} \delta^{(i)} \right] \right] - \mathcal{K} \\
    &= \underset{\policyOcc}{\arg\max} \; \mathbb{E}_{\policyDist} \left[ H(\policy) + \mathbb{E}_{\pi}\left[ \left( \mathcal{U}_{\nicefrac{1}{\policyLR}}^{\rho^{(i)}} \right) \rewardIter{i} \right] \right] - \mathcal{K}
\end{align*}
Finally, we show that $\mathcal{K}$ is in turn upper bounded by $\zeta = \mathbb{E}_{\policyDist} \left[ \text{KL}\left( \policy \left\lvert\right\rvert \policyIter{i} \right) \right]$. As before, we state some intermediate results below.
\begin{align}
    \ThmStepLabel{theorem:policyMD-C}
    \tag{\thethmstep} 
    \text{KL}\left(\overset{\infty}{\underset{t=0}{\sum}} (1 - \gamma) \gamma^t \mu^{\pi}_t \left\lvert\right\rvert \overset{\infty}{\underset{t=0}{\sum}} (1 - \gamma) \gamma^t \mu^{\pi^{(i)}}_t \right) \leq \overset{\infty}{\underset{t=0}{\sum}} (1 - \gamma) \gamma^t \text{KL}\left( \mu^{\pi}_t \left\lvert\right\rvert \mu^{\pi^{(i)}}_t \right) \nonumber
\end{align}
\begin{align*}
    \text{\textit{For any function }} f(\mathbf{s}) \\
    \frac{\gamma}{1 - \gamma} \mathbb{E}_{\policyDist} \left[ f(\mathbf{s}) \right] &= \frac{\gamma}{1 - \gamma} \underset{\mathbf{s}}{\sum} \policyDist \left[ f(\mathbf{s}) \right] \\
    &= \gamma \underset{\mathbf{s}}{\sum} \overset{\infty}{\underset{t=0}{\sum}} \gamma^t \mu^{\pi}_t(\mathbf{s})  \left[ f(\mathbf{s}) \right] \\
    \ThmStepLabel{theorem:policyMD-D}
    \tag{\thethmstep}
    &= \gamma \overset{\infty}{\underset{t=0}{\sum}} \gamma^t \mathbb{E}_{\mu^{\pi}_t} \left[ f(\mathbf{s}) \right]
\end{align*}
Where \Cref{theorem:policyMD-C} follows from the joint convexity of the KL divergence, given $\sum_{t=0}^{\infty}(1-\gamma)\gamma^t=1$. Now, we derive our bound for $\mathcal{K}$. \cite[Lemma~3]{akrour2018model} bound the KL divergence between state distributions of two policies at time $t$.
\begin{align*}
    \text{KL} \left( \mu^{\pi}_t(\mathbf{s}) \left\lvert\right\rvert \mu^{\pi^{(i)}}_t(\mathbf{s}) \right) &\leq \text{KL} \left( \mu^{\pi}_{t-1}(\mathbf{s}) \left\lvert\right\rvert \mu^{\pi^{(i)}}_{t-1}(\mathbf{s}) \right) + \mathbb{E}_{\mu^{\pi}_{t-1}} \left[ \text{KL} \left( \policy \left\lvert\right\rvert \policyIter{i} \right) \right] \\
    &\text{\textit{Given }} \mu^{\pi}_0 = \mu^{\pi^{(i)}}_0 \\
    \text{KL} \left( \mu^{\pi}_1(\mathbf{s}) \left\lvert\right\rvert \mu^{\pi^{(i)}}_1(\mathbf{s}) \right) &\leq 0 + \mathbb{E}_{\mu^{\pi}_{0}} \left[ \text{KL} \left( \policy \left\lvert\right\rvert \policyIter{i} \right) \right] \\
    \text{KL} \left( \mu^{\pi}_2(\mathbf{s}) \left\lvert\right\rvert \mu^{\pi^{(i)}}_2(\mathbf{s}) \right) &\leq 0 + \mathbb{E}_{\mu^{\pi}_{0}} \left[ \text{KL} \left( \policy \left\lvert\right\rvert \policyIter{i} \right) \right] + \mathbb{E}_{\mu^{\pi}_{1}} \left[ \text{KL} \left( \policy \left\lvert\right\rvert \policyIter{i} \right) \right] \\
    &\text{\textit{By induction}} \\
    \text{KL} \left( \mu^{\pi}_t(\mathbf{s}) \left\lvert\right\rvert \mu^{\pi^{(i)}}_t(\mathbf{s}) \right) &\leq \overset{t-1}{\underset{k=0}{\sum}} \mathbb{E}_{\mu^{\pi}_{k}} \left[ \text{KL} \left( \policy \left\lvert\right\rvert \policyIter{i} \right) \right] \\
    &\text{\textit{From \Cref{theorem:policyMD-C}}} \\
    \text{KL} \left( \policyDist \left\lvert\right\rvert \policyDistIter{i} \right) &\leq (1 - \gamma) \overset{\infty}{\underset{t=0}{\sum}} \gamma^t \overset{t-1}{\underset{k=0}{\sum}} \mathbb{E}_{\mu^{\pi}_{k}} \left[ \text{KL} \left( \policy \left\lvert\right\rvert \policyIter{i} \right) \right] \\
    &= \gamma \overset{\infty}{\underset{t=0}{\sum}} \gamma^t \mathbb{E}_{\mu^{\pi}_{t}} \left[ \text{KL} \left( \policy \left\lvert\right\rvert \policyIter{i} \right) \right] \\
    &\text{\textit{From \Cref{theorem:policyMD-D}}} \\
    &= \frac{\gamma}{1 - \gamma} \mathbb{E}_{\policyDist}  \left[ \text{KL} \left( \policy \left\lvert\right\rvert \policyIter{i} \right) \right] \\
    &= \frac{\gamma}{1 - \gamma} \zeta
\end{align*}
Setting the mirror descent step size equal to the inverse step size of the function-space reward update $\omega=\nicefrac{1}{\epsilon}$, we see that the trust region update 
\begin{align*}
\underset{\policy}{\arg \max} &\; \mathbb{E}_{\policyDist} \left[ H(\policy) + \mathbb{E}_\pi \left[ \left( \mathcal{U}_{\epsilon}^{\rho^{(i)}} \right) r^{(i)} \right] \right]
\end{align*}
is analogous to an MD update under the state-action occupancy with a deviation term $\mathcal{K}$ that is in-turn bounded by the enforced trust region bound $\zeta$.
\end{proof}
\section[Convergence Proofs]{\Cref{sec:convergence_analysis} Proofs}
\begin{proof}[\textbf{Proof of \Cref{theorem:policyRegret}}]
Given \Cref{theorem:policyMD}, in the trust region setting, the policy iterates perform entropic mirror descent up to the bounded deviation $\mathcal{K}$. We start with the regret bound for negative Shannon entropy regularized mirror descent \citep{gupta15850f20notes,shalev2025online}
\begin{align*}
  &\text{\textit{let }} \rho_{\pi^{\star}} = \arg \max_{\policyOcc} \mathcal{J}(\policyOcc) \\
  &\max_{\policyOcc} \overset{T}{\underset{i=1}{\sum}} \mathcal{J}(\policyOcc) - \overset{T}{\underset{i=1}{\sum}} \mathcal{J}(\policyOccIter{i}) \\& \leq \policyLR \text{KL}\left( \rho_{\pi^{\star}} \left\lvert\right\rvert \policyOccIter{1} \right) + \frac{\log 2}{2 \policyLR} \overset{T}{\underset{i=1}{\sum}} \left\lVert \nabla_{\rho_{\pi^{(i)}}} \mathcal{J}(\policyOcc) \right\rVert_{\infty}^2 \\
  & \text{\textit{from \Cref{assumption:lipschitz_primal}, }} \left\lVert \nabla_{\rho_{\pi^{(i)}}} \mathcal{J}(\policyOcc) \right\rVert_{\infty} \leq L_{\mathcal{J}} \\
  &\text{\textit{assuming }} \policyOccIter{1} \text{\textit{ is uniform over }} \left\lvert \mathcal{S} \right\rvert \left\lvert \mathcal{A} \right\rvert \\
  & \leq \policyLR \log \left( \left\lvert \mathcal{S} \right\rvert \left\lvert \mathcal{A} \right\rvert \right) + \frac{T L_{\mathcal{J}}^2 \log 2}{2 \policyLR}
\end{align*}
\end{proof}
\begin{proof}[\textbf{Proof of \Cref{lemma:H_bound}}]
We begin by bounding the Euclidean norm $\left\lVert \phi^{(i+1)} - \phi^{(i)} \right\rVert_2$. From \Cref{theorem:policyMD-A} and the bound on $\mathcal{K}$ from \Cref{theorem:policyMD},
\begin{align}
    \text{KL} \left(  \policyOccIter{i+1} \left\lvert \right\rvert \policyOccIter{i} \right) &= \text{KL}\left( \policyDistIter{i+1} \left\lvert\right\rvert \policyDistIter{i} \right) + \underset{\mathbf{s}}{\sum} \policyDistIter{i+1} \text{KL}\left( \policyIter{i+1} \left\lvert\right\rvert \policyIter{i} \right) \nonumber \\
    &\leq \frac{\gamma \zeta}{1 - \gamma} + \zeta \nonumber \\
    &= \frac{\zeta}{1 - \gamma} \label{eq:KL_rho_bound} \\
    \text{\textit{Alternatively, from}} &\text{\textit{ \Cref{assumption:boltzmann_occupancies}}} \nonumber \\
    \text{KL} \left( \policyOccIter{i+1} \left\lvert \right\rvert \policyOccIter{i} \right) &= A(\phi^{(i)}) - A(\phi^{(i+1)}) - \left< \nabla_{\phi}A(\phi^{(i+1)}) , \phi^{(i)} - \phi^{(i+1)} \right> \nonumber \\
    \text{\textit{Where }} A(\phi) = \log Z_{\phi} &\text{\textit{ is the log partition function}} \nonumber \\
    \text{\textit{Since }} \nabla_{\phi}^2 A(\phi) \equiv F({\phi}) \succeq mI \; , \; A &\text{\textit{ is }} m \text{\textit{-strongly convex (\Cref{prop:vanillaGradient});}} \nonumber \\
    \text{\textit{assuming \Cref{assumption:uniform_covariate_conditioning}}} &\text{\textit{ also holds on the segment between }} \phi^{(i)} \text{\textit{ and }} \phi^{(i+1)} \nonumber \\ 
    \text{\textit{from the definition}} &\text{\textit{ of strong convexity}} \nonumber \\
    \frac{m}{2} \left\lVert \phi^{(i+1)} - \phi^{(i)} \right\rVert_2^2 &\leq A(\phi^{(i)}) - A(\phi^{(i+1)}) - \left< \nabla_{\phi}A(\phi^{(i+1)}) , \phi^{(i)} - \phi^{(i+1)} \right> \nonumber \\
    &\text{\textit{from \Cref{eq:KL_rho_bound}}} \nonumber \\
    \frac{m}{2} \left\lVert \phi^{(i+1)} - \phi^{(i)} \right\rVert_2^2 &\leq \frac{\zeta}{1 - \gamma} \nonumber \\
    \left\lVert \phi^{(i+1)} - \phi^{(i)} \right\rVert_2 &\leq \sqrt{\frac{2 \zeta}{m(1 - \gamma)}} \label{eq:rho_bound}
\end{align}
Below we bound the norm $\left\lVert H^{(i+1)} - H^{(i)} \right\rVert_2$. From the definition of $H^{(i)}$ and the triangle inequality
\begin{align*}
    \left\lVert H^{(i+1)} - H^{(i)} \right\rVert_2 &\leq \overset{1}{\underset{0}{\int}} \left\lVert F(\phi^{(i+1)}(\tau)) - F(\phi^{(i)}(\tau)) \right\rVert_2 d\tau \\
    &\text{\textit{from \Cref{assumption:lipschitz_FIM}}} \\
    &\leq L_F \overset{1}{\underset{0}{\int}} \left\lVert \phi^{(i+1)}(\tau) - \phi^{(i)}(\tau) \right\rVert_2 d\tau \\
    &\text{\textit{from the definition of }} \phi^{(i)}(\tau) \\
    &= L_F \overset{1}{\underset{0}{\int}} \left\lVert (1 - \tau)\left(\hat{\phi}^{(i+1)} - \hat{\phi}^{(i)}\right) + \tau \left( \phi^{(i+1)} - \phi^{(i)} \right) \right\rVert_2 d\tau \\
    &\text{\textit{triangle inequality and integrating over }} \tau \\
    &\leq L_F \left( \frac{1}{2} \left\lVert \hat{\phi}^{(i+1)} - \hat{\phi}^{(i)} \right\rVert_2 + \frac{1}{2} \left\lVert \phi^{(i+1)} - \phi^{(i)} \right\rVert_2 \right) \\
    &\text{\textit{from the definition of }} \hat{\phi}^{(i)} \text{\textit{ (\Cref{eq:phi_hat})}} \\
    &= \frac{L_F}{2} \left( \frac{1}{\beta} \left\lVert \theta^{(i+1)} - \theta^{(i)} \right\rVert_2 + \left\lVert \phi^{(i+1)} - \phi^{(i)} \right\rVert_2 \right) \\
    &= \frac{L_F}{2} \left( \frac{\epsilon}{\beta} \left\lVert \delta^{(i)} \right\rVert_2 + \left\lVert \phi^{(i+1)} - \phi^{(i)} \right\rVert_2 \right) \\
    &\text{\textit{from \Cref{lemma:pseudograd}}} \\
    &\leq \frac{L_F}{2} \left( \frac{\epsilon}{m} \left\lVert \nabla_{\theta} \mathcal{G} \right\rVert_2 + \left\lVert \phi^{(i+1)} - \phi^{(i)} \right\rVert_2 \right) \\
    &\text{\textit{from \Cref{eq:rho_bound,assumption:lipschitzness_and_boundedness} (bounded gradient from Lipschitzness)}} \\
    &\leq \frac{L_F}{2} \left( \frac{\epsilon L_{\mathcal{G}}}{m} + \sqrt{\frac{2 \zeta}{m(1 - \gamma)}} \right)
\end{align*}
\end{proof}
\begin{proof}[\textbf{Proof of \Cref{theorem:rewardRegret}}]
Given \Cref{theorem:rewardMD}, we start with regret bound for mirror descent under a time-varying proximal term \citep{duchi2011adaptive}
\begin{align*}
    &\overset{T}{\underset{i=1}{\sum}} \mathcal{G}(\rewardIter{i}) - \min_{\theta} \overset{T}{\underset{i=1}{\sum}} \mathcal{G}(r_{\theta}) \\& \leq \frac{D_{\Omega^{(1)}}(\theta^{\star}, \theta^{(1)})}{\epsilon} + \frac{1}{\epsilon} \overset{T}{\underset{i=1}{\sum}} \left[ D_{\Omega^{(i+1)}}(\theta^{\star}, \theta^{(i+1)}) - D_{\Omega^{(i)}}(\theta^{\star}, \theta^{(i+1)}) \right] + \frac{\epsilon}{2} \overset{T}{\underset{i=1}{\sum}} \left\lVert \nabla_{\theta^{(i)}} \mathcal{G} \right\rVert_{\Omega^{\star (i)}}^2 \\
    &\text{\textit{where }} \left\lVert\cdot\right\rVert_{\Omega^{\star (i)}} \text{\textit{ is the dual norm of }} \left\lVert\cdot\right\rVert_{\Omega^{(i)}} \text{\textit{w.r.t which }} \Omega^{(i)} \text{\textit{ is 1-strongly convex}} \\
    &\text{\textit{given a squared primal norm }} \frac{1}{\beta}\left< \nabla_{\theta^{(i)}} \mathcal{G} , H^{(i)} \nabla_{\theta^{(i)}} \mathcal{G} \right> \text{\textit{ the squared dual norm is }} \beta \left< \nabla_{\theta^{(i)}} \mathcal{G} , \left( H^{(i)} \right)^{-1} \nabla_{\theta^{(i)}} \mathcal{G} \right> \\
    &= \frac{\left\lVert \theta^{\star} - \theta^{(1)} \right\rVert_{H^{(1)}}^2}{2 \epsilon \beta}  + \frac{1}{2 \epsilon \beta} \overset{T}{\underset{i=1}{\sum}} \left[ \left\lVert \theta^{\star} - \theta^{(i+1)} \right\rVert_{H^{(i+1)}}^2 - \left\lVert \theta^{\star} - \theta^{(i+1)} \right\rVert_{H^{(i)}}^2 \right] + \frac{\epsilon \beta}{2} \overset{T}{\underset{i=1}{\sum}} \left< \nabla_{\theta^{(i)}} \mathcal{G} , \left( H^{(i)} \right)^{-1} \nabla_{\theta^{(i)}} \mathcal{G} \right> \\
    &\text{\textit{from \Cref{assumption:lipschitzness_and_boundedness,assumption:uniform_covariate_conditioning,assumption:bounded_reward_domain}} and Cauchy-Schwarz} \\
    &\leq \frac{\left\lVert \theta^{\star} - \theta^{(1)} \right\rVert_{H^{(1)}}^2}{2 \epsilon \beta}  + \frac{1}{2 \epsilon \beta} \overset{T}{\underset{i=1}{\sum}} \left[ \left\lVert \theta^{\star} - \theta^{(i+1)} \right\rVert_2^2 \left\lVert H^{(i+1)} - H^{(i)} \right\rVert_2 \right] + \frac{\epsilon \beta}{2 m} \overset{T}{\underset{i=1}{\sum}} \left\lVert \nabla_{\theta^{(i)}} \mathcal{G} \right\rVert_2^2 \\
    &\leq \frac{M B^2}{2 \epsilon \beta} + \frac{B^2}{2 \epsilon \beta} \overset{T}{\underset{i=1}{\sum}} \left[ \left\lVert H^{(i+1)} - H^{(i)} \right\rVert_2 \right] + \frac{\epsilon \beta L_{\mathcal{G}}^2 T}{2 m} \\
    &\text{\textit{from \Cref{lemma:H_bound}}} \\
    &\leq \frac{M B^2}{2 \epsilon \beta} + \frac{T B^2 L_F}{4 \epsilon \beta} \left( \frac{\epsilon L_{\mathcal{G}}}{m} + \sqrt{\frac{2 \zeta}{m(1 - \gamma)}} \right) + \frac{\epsilon \beta L_{\mathcal{G}}^2 T}{2 m}
\end{align*}
\end{proof}
\begin{proof}[\textbf{Proof of \Cref{theorem:SPR_convergence}}]
From \Cref{fact:strong_duality}, and since $r^\star$ is linear in features, we know that the IRL problem has strong duality, i.e\@. $\min_{\theta} \mathcal{G}(r_{\theta}) = \max_{\rho_{\pi}} \mathcal{J}(\rho_{\pi}(\mathbf{s},\mathbf{a}))$. Let $R_{\pi}$ and $R_{r}$ denote the regret bounds in \Cref{theorem:policyRegret,theorem:rewardRegret}. Since both the policy and reward players are regret-minimizing algorithms that guarantee bounded regret, we can apply standard Nash equilibrium theory \citep[Section~2.1; Section~2]{farina2019optimistic,grand2021conic} to get
\begin{align*}
    \text{SPR}_{T} &= \mathcal{G}(\Bar{r}) - \mathcal{J}(\Bar{\rho}(\mathbf{s},\mathbf{a})) \leq \frac{1}{T} (R_{\pi} + R_{r}) \\
    &\leq \frac{1}{T} \left[ \left( \policyLR \log \left( \left\lvert \mathcal{S} \right\rvert \left\lvert \mathcal{A} \right\rvert \right) + \frac{T L_{\mathcal{J}}^2 \log 2}{2 \policyLR} \right) + \left( \frac{M B^2}{2 \epsilon \beta} + \frac{T B^2 L_F}{4 \epsilon \beta} \left( \frac{\epsilon L_{\mathcal{G}}}{m} + \sqrt{\frac{2 \zeta}{m(1 - \gamma)}} \right) + \frac{\epsilon \beta L_{\mathcal{G}}^2 T}{2 m} \right) \right] \\
    &= \frac{\policyLR}{T} \log \left( \left\lvert \mathcal{S} \right\rvert \left\lvert \mathcal{A} \right\rvert \right) + \frac{L_{\mathcal{J}}^2 \log 2}{2 \policyLR} + \frac{M B^2}{2 \epsilon \beta T} + \frac{B^2 L_F L_{\mathcal{G}}}{4 \beta m} + \frac{B^2 L_F}{4 \epsilon \beta} \sqrt{\frac{2 \zeta}{m(1 - \gamma)}} + \frac{\epsilon \beta L_{\mathcal{G}}^2}{2 m} \\
    &\text{\textit{choose iteration dependent learning rates }} \policyLR = \sqrt{T} \text{\textit{ and }} \epsilon = \frac{1}{\sqrt{T}} \text{\textit{, with trust-region tolerance }} \zeta = \frac{1}{T^2} \\
    \text{SPR}_{T} &\leq \frac{1}{\sqrt{T}} \log \left( \left\lvert \mathcal{S} \right\rvert \left\lvert \mathcal{A} \right\rvert \right) + \frac{L_{\mathcal{J}}^2 \log 2}{2 \sqrt{T}} + \frac{M B^2}{2 \beta \sqrt{T}} + \frac{B^2 L_F L_{\mathcal{G}}}{4 \beta m} + \frac{B^2 L_F}{4 \beta \sqrt{T}} \sqrt{\frac{2}{m(1 - \gamma)}} + \frac{\beta L_{\mathcal{G}}^2}{2 m \sqrt{T}} \\
    &\text{\textit{let }} C_{\text{SPR}} = \log \left( \left\lvert \mathcal{S} \right\rvert \left\lvert \mathcal{A} \right\rvert \right) + \frac{L_{\mathcal{J}}^2 \log 2}{2} + \frac{M B^2}{2 \beta} + \frac{B^2 L_F}{4 \beta} \sqrt{\frac{2}{m(1 - \gamma)}} + \frac{\beta L_{\mathcal{G}}^2}{2 m} \\
    \text{SPR}_{T} &\leq \frac{B^2 L_F L_{\mathcal{G}}}{4 \beta m} + \frac{C_{\text{SPR}}}{\sqrt{T}}. \quad \text{\textit{Therefore, for any }} \xi > \frac{B^2 L_F L_{\mathcal{G}}}{4 \beta m} \text{\textit{, it suffices that}} \\
    \sqrt{T} &\geq \frac{C_{\text{SPR}}}{\xi - \frac{B^2 L_F L_{\mathcal{G}}}{4 \beta m}} \\
    T &= \mathcal{O}\left( \frac{C_{\text{SPR}}^2}{\left( \xi - \frac{B^2 L_F L_{\mathcal{G}}}{4 \beta m} \right)^2} \right)
\end{align*}
\end{proof}


\newpage
\section*{NeurIPS Paper Checklist}

\begin{enumerate}

\item {\bf Claims}
    \item[] Question: Do the main claims made in the abstract and introduction accurately reflect the paper's contributions and scope?
    \item[] Answer: \answerYes{} 
    \item[] Justification: The abstract and the introduction claim that (1) the fisher information along the e-geodesic directly affects the dual gradient of IRL, which can be exploited without computational overhead, (2) that the resulting update performs approximate entropic mirror descent in occupancy space, (3) that we provide convergence and regret bounds of the resulting update. These claims are directly supported by our theorems and lemmas.

\item {\bf Limitations}
    \item[] Question: Does the paper discuss the limitations of the work performed by the authors?
    \item[] Answer: \answerYes{} 
    \item[] Justification: We discuss the following limitations in \Cref{sec:discussion}: (i) Our derivations do not take errors due to function approximation and limited data into account, (ii) we assume that there exists a finite feature vector in which the occupancies and rewards are linear, (iii) our reward-player regret bound depends on \Cref{assumption:bounded_reward_domain}, which may not always hold in practice.

\item {\bf Theory assumptions and proofs}
    \item[] Question: For each theoretical result, does the paper provide the full set of assumptions and a complete (and correct) proof?
    \item[] Answer: \answerYes{} 
    \item[] Justification: All claims made in the paper are accompanied by a corresponding proof. We state all assumptions and remarks for these theoretical results in the main text (and some in the appendix). 

    \item {\bf Experimental result reproducibility}
    \item[] Question: Does the paper fully disclose all the information needed to reproduce the main experimental results of the paper to the extent that it affects the main claims and/or conclusions of the paper (regardless of whether the code and data are provided or not)?
    \item[] Answer: \answerYes{} 
    \item[] Justification: We elaborate on our experimental setup in \Cref{appendix:experiment_details}. Code is also provided in supplementary material.

\item {\bf Open access to data and code}
    \item[] Question: Does the paper provide open access to the data and code, with sufficient instructions to faithfully reproduce the main experimental results, as described in supplemental material?
    \item[] Answer: \answerYes{} 
    \item[] Justification: Our code is available at \href{https://github.com/anishhdiwan/geometric-perspective}{https://github.com/anishhdiwan/geometric-perspective}.

\item {\bf Experimental setting/details}
    \item[] Question: Does the paper specify all the training and test details (e.g., data splits, hyperparameters, how they were chosen, type of optimizer) necessary to understand the results?
    \item[] Answer: \answerYes{} 
    \item[] Justification: Our experiments mainly help illustrate the proposed theoretical mechanism. We provide these details in \Cref{appendix:experiment_details} with specific hyperparameters for IRL provided in the supplementary material.

\item {\bf Experiment statistical significance}
    \item[] Question: Does the paper report error bars suitably and correctly defined or other appropriate information about the statistical significance of the experiments?
    \item[] Answer: \answerNA{} 
    \item[] Justification: Our experiments mainly help illustrate the proposed theoretical mechanism. Hence, error bars are not relevant here.

\item {\bf Experiments compute resources}
    \item[] Question: For each experiment, does the paper provide sufficient information on the computer resources (type of compute workers, memory, time of execution) needed to reproduce the experiments?
    \item[] Answer: \answerNA{} 
    \item[] Justification: Our experiments mainly help illustrate the proposed theoretical mechanism. The required compute is negligible, however we briefly mention this in \Cref{appendix:experiment_details}.
    
\item {\bf Code of ethics}
    \item[] Question: Does the research conducted in the paper conform, in every respect, with the NeurIPS Code of Ethics \url{https://neurips.cc/public/EthicsGuidelines}?
    \item[] Answer: \answerYes{} 
    \item[] Justification: We checked the guidelines and we are compliant with them.

\item {\bf Broader impacts}
    \item[] Question: Does the paper discuss both potential positive societal impacts and negative societal impacts of the work performed?
    \item[] Answer: \answerYes{} 
    \item[] Justification: Although our contribution is theoretical, advances in IRL and imitation learning may affect how AI systems learn from human behavior. These methods could support safer alignment with expert intent, but they may also reproduce or amplify biases, errors, or unsafe behavior present in the demonstrations, underscoring the importance of demonstration quality and appropriate oversight.
    
\item {\bf Safeguards}
    \item[] Question: Does the paper describe safeguards that have been put in place for responsible release of data or models that have a high risk for misuse (e.g., pre-trained language models, image generators, or scraped datasets)?
    \item[] Answer: \answerNA{} 
    \item[] Justification: As a theoretical contribution, this paper poses no such risks.

\item {\bf Licenses for existing assets}
    \item[] Question: Are the creators or original owners of assets (e.g., code, data, models), used in the paper, properly credited and are the license and terms of use explicitly mentioned and properly respected?
    \item[] Answer: \answerNA{} 
    \item[] Justification: Apart from common Python libraries like Numpy/Jax, we do not use any other code, data, models, or other assets.

\item {\bf New assets}
    \item[] Question: Are new assets introduced in the paper well documented and is the documentation provided alongside the assets?
    \item[] Answer: \answerYes{} 
    \item[] Justification: In order to run the experiments, we implemented a detailed code-base. This asset is available at \href{https://github.com/anishhdiwan/geometric-perspective}{https://github.com/anishhdiwan/geometric-perspective}. The zip file includes a README explaining how to employ the code-base.

\item {\bf Crowdsourcing and research with human subjects}
    \item[] Question: For crowdsourcing experiments and research with human subjects, does the paper include the full text of instructions given to participants and screenshots, if applicable, as well as details about compensation (if any)? 
    \item[] Answer: \answerNA{} 
    \item[] Justification: The paper does not involve crowdsourcing nor research with human subjects.

\item {\bf Institutional review board (IRB) approvals or equivalent for research with human subjects}
    \item[] Question: Does the paper describe potential risks incurred by study participants, whether such risks were disclosed to the subjects, and whether Institutional Review Board (IRB) approvals (or an equivalent approval/review based on the requirements of your country or institution) were obtained?
    \item[] Answer: \answerNA{} 
    \item[] Justification: The paper does not involve crowdsourcing nor research with human subjects.

\item {\bf Declaration of LLM usage}
    \item[] Question: Does the paper describe the usage of LLMs if it is an important, original, or non-standard component of the core methods in this research? Note that if the LLM is used only for writing, editing, or formatting purposes and does \emph{not} impact the core methodology, scientific rigor, or originality of the research, declaration is not required.
    \item[] Answer: \answerNA{} 
    \item[] Justification: The core method development in this research does not involve LLMs as any important, original, or non-standard components.    

\end{enumerate}

\end{document}